\documentclass[sigconf]{acmart}
\usepackage{colortbl}
\definecolor{cartblue}{RGB}{220,235,252}
\definecolor{cartanchor}{RGB}{245,245,230}
\AtBeginDocument{%
  }

\setcopyright{acmlicensed}
\copyrightyear{2026}
\acmYear{2026}
\acmDOI{XXXXXXX.XXXXXXX}
\acmConference[ACM MM 2026]{The 34th ACM International Conference on Multimedia}{October 2026}{Dublin, Ireland}

\newtheorem{assumption}{Assumption}
\newtheorem{definition}{Definition}
\newtheorem{remark}{Remark}
\newtheorem{theorem}{Theorem}

\newtheorem{corollary}{Corollary}
\newtheorem{proposition}{Proposition}
\usepackage{booktabs}
\usepackage{multirow}
\usepackage{makecell}
\usepackage{algorithm}
\usepackage{algpseudocode}
\begin{document}

\title{Constraint-Anchored Reasoning Traces}

\author{Zehua Cheng}
\affiliation{
  \institution{Department of Computer Science\\University of Oxford}
  \country{United Kingdom}
}
\email{zehua.cheng@cs.ox.ac.uk}
\author{Wei Dai and Jiahao Sun}
\affiliation{
  \institution{FLock.io}
  \country{United Kingdom}
}
\email{sun@flock.io}

\renewcommand{\shortauthors}{Cheng et al.}

\begin{abstract}
  Autoregressive multimodal large language models (MLLMs) suffer from error snowballing: a single incorrect inference early in a chain-of-thought (CoT) trace corrupts all downstream reasoning. We find that in state-of-the-art open-source MLLMs, once the first error occurs, the reasoning cascades into failure across all remaining steps in 65\% of such cases (a metric we term the \emph{snowball rate}). Existing mitigations---sampling multiple chains, post-hoc self-verification, or full program synthesis---either lack symbolic grounding, catch errors too late, or sacrifice the flexibility of natural language reasoning. We propose \textbf{Constraint-Anchored Reasoning Traces (CART)}, a neuro-symbolic framework that trains MLLMs to interleave natural language reasoning steps with symbolic constraint assertions: lightweight, machine-checkable statements about visual content (e.g., \texttt{count(red\_objects)~=~3}). A dual-pronged Constraint Propagation Module---combining a learned neural grounding head with Boolean Constraint Propagation---continuously verifies these anchors against extracted visual features and checks their mutual logical consistency. When a contradiction is detected, a backtrack controller halts generation and reverts to the last consistent checkpoint, preventing error propagation. A variable-frequency emission mechanism allows the model to adaptively control anchor density, avoiding trace bloat. We construct 218K training instances by augmenting GQA, CLEVR-CoGenT, and VCR with ground-truth constraint annotations derived from scene graphs, and fine-tune open-source MLLMs (LLaVA-NeXT, Qwen2-VL) via LoRA. On five benchmarks, CART reduces the snowball rate from 0.65 to 0.14, improves GQA accuracy by +4.6 percentage points over training-only baselines, and achieves 89.1 F1 on POPE---all with at most 18\% inference overhead.
\end{abstract}



\maketitle

\section{Introduction}

Multimodal large language models (MLLMs) have demonstrated remarkable capabilities across vision-language tasks, yet their autoregressive generation paradigm harbors a critical structural vulnerability: \textit{error snowballing}. When a model commits an incorrect inference early in a chain-of-thought (CoT) reasoning trace, every subsequent step conditions on this erroneous prefix, causing downstream reasoning to cascade into compounding failures~\cite{zhang2023language,wei2022chain}. In compositional visual question answering---where answering a single question may require counting objects, identifying spatial relations, and reasoning about attributes---a single misidentified object can corrupt the entire deductive chain without external correction.

Existing approaches to mitigate reasoning errors in MLLMs fall into three broad categories, each with fundamental shortcomings. \textit{Sampling-based methods}~\cite{wang2023selfconsistency,yao2023tree} generate multiple candidate chains and select answers by majority vote or heuristic scoring, but provide no mechanism to identify \textit{which} step failed or \textit{why}---they merely hope that the correct chain appears among the samples. \textit{Post-hoc self-verification}~\cite{shinn2023reflexion,madaan2023selfrefine} prompts the model to critique its own output, but this verbal self-assessment lacks grounding in the visual input; the same model that produced the error is unlikely to reliably detect it through introspection alone. \textit{Program synthesis}~\cite{gupta2023visual,suris2023vipergpt} sidesteps natural language reasoning entirely by generating executable code, achieving strong symbolic guarantees but sacrificing the interpretability and flexibility of natural language traces---a critical limitation when the reasoning domain is not reducible to a fixed API.

We observe that what is needed is a \textit{principled middle ground}: a framework that preserves the flexibility of natural language reasoning while injecting verifiable symbolic checkpoints that can catch errors \textit{as they occur} and trigger corrective action \textit{before} they propagate. To this end, we propose \textbf{Constraint-Anchored Reasoning Traces (CART)}, a framework that trains MLLMs to interleave standard reasoning steps with \textbf{symbolic constraint assertions}---lightweight, machine-checkable statements about visual content (e.g., \texttt{count(red\_objects) = 3}, \texttt{left\_of(dog, cat) = True}). A dedicated \textbf{Constraint Propagation Module (CPM)} continuously verifies these anchors against extracted visual features via a neural grounding head and checks their mutual logical consistency via Boolean Constraint Propagation (BCP). When a contradiction is detected---either between an anchor and the image or among accumulated anchors---generation halts and a \textbf{backtrack controller} reverts the trace to the last consistent checkpoint, preventing error propagation.

A key design insight of CART is \textit{variable-frequency anchor emission}: rather than rigidly alternating between reasoning steps and constraints, the model learns to emit zero, one, or multiple anchors after each step, with an explicit null token when no meaningful constraint applies. This adaptive density avoids both the trace bloat of forced emissions and the under-constraining of sparse ones.

We train CART by augmenting existing visual reasoning datasets (GQA~\cite{hudson2019gqa}, CLEVR-CoGenT~\cite{johnson2017clevr}, VCR~\cite{zellers2019recognition}) with ground-truth constraint annotations derived from scene graphs and structured programs, yielding 218K training instances. Fine-tuning is performed via LoRA on open-source MLLMs (LLaVA-NeXT~\cite{zhang2024llavanext-video}, Qwen2-VL~\cite{wang2024qwen2}), jointly optimizing a task loss, a constraint grounding loss, and a novel backtrack-aware loss that teaches the model to recover from simulated errors.

Our contributions are summarized as follows:
\begin{itemize}
  \item We formalize the error snowballing problem in autoregressive multimodal reasoning and propose CART, a neuro-symbolic framework that interleaves natural language reasoning with verifiable symbolic constraint anchors and a runtime constraint propagation module.
  \item We introduce a variable-frequency anchor emission mechanism, a dual-pronged verification engine combining neural visual grounding with symbolic Boolean Constraint Propagation, and a backtrack controller with formal error-reduction guarantees (Theorem~\ref{thm:snowball-rate-reduction}).
  \item We define three novel diagnostic metrics---snowball rate, constraint violation rate, and error attribution precision---that directly measure error propagation behavior beyond standard end-task accuracy.
  \item Extensive experiments on five benchmarks (GQA, CLEVR-CoGenT, VCR, MM-Vet, POPE) demonstrate that CART achieves state-of-the-art accuracy while reducing the snowball rate from 0.65 to 0.14, with only 10--18\% inference overhead.
\end{itemize}

\section{Related Work}

\textbf{Chain-of-thought reasoning in MLLMs.} Chain-of-thought prompting~\cite{wei2022chain} and its variants have become the de facto approach for eliciting multi-step reasoning from large language models. Zero-shot CoT~\cite{kojima2022large} and few-shot prompting strategies improve performance on arithmetic and commonsense benchmarks, while multimodal extensions~\cite{zhang2023multimodal,lu2022learn} adapt these techniques to vision-language tasks. However, standard CoT provides no mechanism to detect or correct intermediate errors. Tree-of-Thoughts~\cite{yao2023tree} explores multiple reasoning branches via breadth-first or depth-first search with LLM-based evaluation, but the evaluator shares the same failure modes as the generator. Self-Consistency~\cite{wang2023selfconsistency} marginalizes over sampled chains by majority voting, improving robustness statistically but offering no per-chain error correction. DDCoT~\cite{zheng2023ddcot} decomposes questions into sub-questions for divide-and-conquer reasoning but still relies on the model's internal representations for correctness. CART differs from all of these by introducing \textit{externally verifiable} symbolic checkpoints that ground intermediate reasoning in the visual input, enabling targeted correction rather than statistical aggregation.

\textbf{Neuro-symbolic visual reasoning.} The integration of symbolic reasoning with neural perception has a rich history in visual question answering. Neural Module Networks~\cite{andreas2016neural,hu2017learning} compose differentiable modules according to syntactic parses of questions, while NS-VQA~\cite{yi2018neural} uses a scene parser to construct symbolic representations for program execution. More recently, VisProg~\cite{gupta2023visual} and ViperGPT~\cite{suris2023vipergpt} leverage LLMs to generate executable Python programs that invoke specialized vision modules. Faithful CoT~\cite{lyu2023faithful} translates reasoning chains into symbolic programs for verification. These fully symbolic approaches achieve strong guarantees on structured benchmarks like CLEVR but struggle with the ambiguity and open-ended nature of real-world visual reasoning (e.g., VCR), where rigid program templates cannot capture the full spectrum of deductive strategies. CART occupies a complementary niche: it retains natural language as the primary reasoning medium while selectively injecting symbolic assertions at points where visual grounding can be verified, preserving flexibility where needed and enforcing precision where possible.

\textbf{Error correction and self-refinement.} Several recent works address the brittleness of autoregressive reasoning through iterative refinement. Reflexion~\cite{shinn2023reflexion} maintains an episodic memory of verbal self-reflections to improve subsequent attempts, while Self-Refine~\cite{madaan2023selfrefine} iterates between generation and self-critique within a single inference pass. In the visual domain, visual verification approaches~\cite{deng2024seeing} prompt MLLMs to re-examine the image after generating an answer. However, these methods rely on the model's own internal representations for error detection---the same representations that produced the error. CART introduces a \textit{partially external} verification loop: the neural grounding head is anchored to \emph{independently extracted} visual evidence (region features from a frozen Grounding DINO~\cite{liu2024grounding} detector) rather than relying solely on the generator's own attention over the image, and BCP provides a non-differentiable, sound consistency check over the accumulated anchors. We note that the grounding head is not fully independent of the generator---it also consumes the anchor's LLM hidden-state embedding (\S\ref{sec:entity-group-mech})---but grounding the check in external visual features and a symbolic consistency test still mitigates the purely self-referential failure mode of methods that verify using only the generator's internal representations.


\begin{figure*}
  \includegraphics[width=\textwidth]{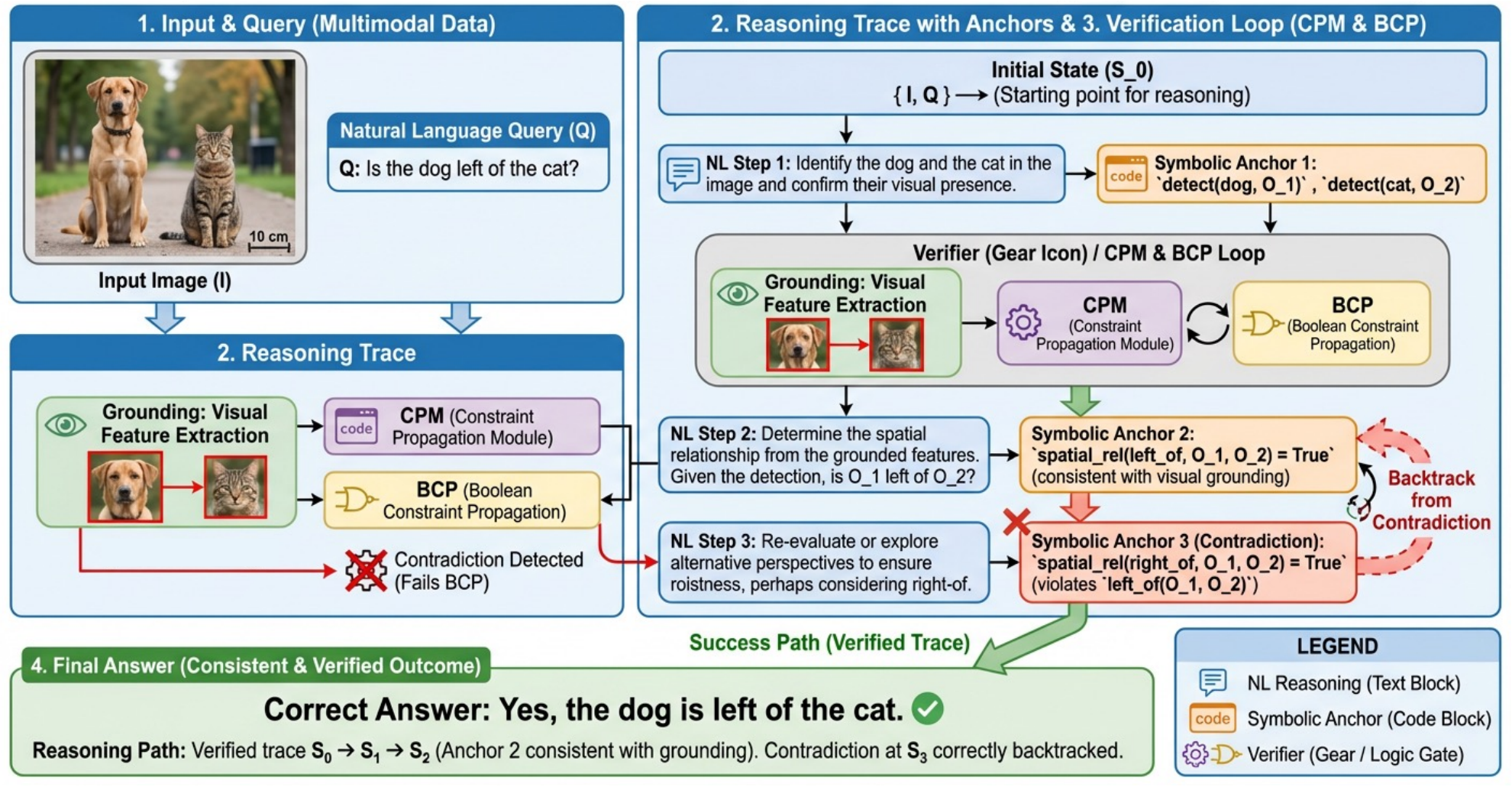}
  \caption{Overview of the CART framework. Given an image and question, the MLLM generates interleaved reasoning steps and symbolic constraint anchors. The Constraint Propagation Module (CPM) verifies each anchor via a neural grounding head and Boolean Constraint Propagation. When a violation is detected, the backtrack controller reverts to the last consistent checkpoint and resumes generation with a \texttt{<RETRY>} token.}
  \Description{Architecture diagram showing the CART pipeline: an MLLM generates reasoning steps interleaved with constraint anchors, which are verified by the CPM (neural grounding head and BCP). A backtrack controller handles detected violations.}
\end{figure*}

\section{Methodology}
We present \textbf{Constraint-Anchored Reasoning Traces (CART)}, a framework that augments autoregressive multimodal reasoning with interleaved symbolic constraint assertions and a runtime constraint propagation module capable of detecting inconsistencies and triggering targeted backtracking. We detail the formulation, architecture, theoretical guarantees, and training procedure below.

\subsection{Problem Formulation\label{sec:problem-formulation}}
Let $\mathcal{I}$ denote the space of images and $\mathcal{Q}$ the space of natural-language questions. A multimodal large language model (MLLM) $\mathcal{M}_\theta$ parameterized by $\theta$ produces, given an image $I \in \mathcal{I}$ and question $q \in \mathcal{Q}$, an answer $a$ via an autoregressive chain-of-thought (CoT) trace $\mathbf{r} = (r_1, r_2, \ldots, r_T)$, where each $r_t$ is a reasoning step (a span of tokens). The final answer $a = f_{\text{ans}}(\mathbf{r})$ is extracted from the terminal step. We denote the visual feature bank extracted by the MLLM's frozen vision encoder as $\mathbf{V} = \phi(I) \in \mathbb{R}^{N \times d_v}$, where $N$ is the number of spatial tokens and $d_v$ is the per-token feature dimension.

\textbf{Error snowballing.} We define an \textbf{error event} at step $t$ as $e_t = \mathbb{1}[\text{sem}(r_t) \neq \text{sem}(r_t^*)]$, where $\text{sem}(\cdot)$ extracts the semantic proposition of a step and $r_t^*$ is the ground-truth counterpart. The \textbf{snowball rate} $\mathcal{S}$ is the conditional probability that all subsequent steps are erroneous given a first error:
$$\mathcal{S} = \mathbb{P}\!\Big[\bigcap_{t'=t+1}^{T} \{e_{t'}=1\} \;\Big|\; e_t = 1,\; \forall_{t''<t}\; e_{t''}=0\Big].$$

Standard autoregressive generation exhibits high $\mathcal{S}$ because the model conditions on its own (erroneous) prefix with no corrective mechanism.

\textbf{Objective.} We seek to learn parameters $\theta^*$ and a constraint-checking mechanism $\mathcal{C}$ that jointly minimize task loss while reducing $\mathcal{S}$:

\begin{equation}
  \theta^* = \arg\min_\theta \; \underbrace{\mathcal{L}_{\text{task}}(\theta)}_{\text{answer accuracy}} + \lambda_c \underbrace{\mathcal{L}_{\text{cst}}(\theta)}_{\text{constraint grounding}} + \lambda_b \underbrace{\mathcal{L}_{\text{bt}}(\theta)}_{\text{backtrack policy}},
\end{equation}
where $\lambda_c, \lambda_b > 0$ are balancing coefficients. Each loss term is defined in Appendix~\ref{appx:training-optim-procedure}.
\subsection{Constraint-Anchored Reasoning Traces}
CART introduces three interlocking components: (i) \textbf{variable-frequency anchor-interleaved generation}, (ii) a \textbf{constraint propagation module (CPM)}, and (iii) a \textbf{backtrack controller}. We describe each in turn and justify every design choice.

\subsubsection{Variable-Frequency Anchor-Interleaved Generation}

We augment the reasoning trace vocabulary with a distinguished set of \textbf{constraint anchors}. Formally, an anchor $c$ is a symbolic assertion drawn from a domain-specific constraint language $\mathcal{L}_c$:

$$c \;=\; \texttt{pred}(\texttt{arg}_1, \ldots, \texttt{arg}_k) \bowtie v, \quad \bowtie \;\in\; \{=, \neq, <, >, \leq, \geq\},$$
where $\texttt{pred}$ is a visual predicate (e.g., \texttt{count}, \texttt{color}, \texttt{spatial\_rel}), $\texttt{arg}_i$ are grounded entity references, and $v$ is a value. Examples include \texttt{count(red\_objects) = 3} and \texttt{left\_of(dog, cat) = True}. The language $\mathcal{L}_c$ is intentionally kept lightweight—propositional with bounded-arity predicates and finite domains—so that checking is efficient (see Proposition~\ref{prop:bcp-tractability}).

\textbf{Variable-frequency emission}. Rather than enforcing a rigid one-to-one alternation between reasoning steps and anchors, we allow the model to emit \textbf{zero, one, or multiple} anchors after any reasoning step. Anchors are delineated by \texttt{<CON>} and \texttt{</CON>} delimiter tokens; the model may emit multiple such delimited spans consecutively, or none at all. When no meaningful constraint applies to a reasoning step, the model emits an explicit \texttt{<NULL\_CON/>} token. Formally, the mixed trace has the structure:
$$\mathbf{r} = \big(r_1, \mathbf{c}_1, \; r_2, \mathbf{c}_2, \; \ldots, \; r_T, \mathbf{c}_T\big),$$
where $\mathbf{c}_t \in \mathcal{L}_c^{*} \cup \{\texttt{<NULL\_CON/>}\}$ denotes a possibly-empty sequence of anchors following step $r_t$. Let $\mathcal{A} = \{(t,j) : c_{t,j} \in \mathbf{c}_t,\; c_{t,j} \neq \texttt{<NULL\_CON/>}\}$ index all non-null anchor emissions across the trace.
\noindent \textit{Justification.} Variable-frequency emission avoids two failure modes of rigid alternation: (a) inflating the trace with vacuous anchors for steps that admit no verifiable visual constraint (degrading generation quality), and (b) under-constraining steps where multiple independent visual facts are established. The \texttt{<NULL\_CON/>} token is necessary because simply omitting the anchor field would create an ambiguous parse: the delimiter-based parser cannot distinguish ``no anchor emitted'' from ``anchor emission failed.'' Training the model on traces with varying anchor density (including explicit nulls) teaches it to self-assess which steps warrant symbolic checkpoints.

\subsubsection{Entity Grounding Mechanism}\label{sec:entity-group-mech}
Anchors reference entities by natural-language name strings (e.g., `dog`, `red cylinder`). Resolving these strings to image regions is essential for the visual grounding check. We use Grounding DINO~\citep{liu2024grounding} as a frozen, off-the-shelf entity grounding module. Specifically:
\begin{itemize}
  \item For each anchor $c_{t,j}$ with entity arguments $(\texttt{arg}_1, \ldots, \texttt{arg}_k)$, each argument string is passed to Grounding DINO along with the image $I$, producing a set of candidate bounding boxes with confidence scores.
  \item The top-1 box per argument (confidence threshold $\geq 0.3$) is projected onto the spatial token grid of the vision encoder, yielding a subset of spatial token indices $\mathcal{R}_i \subseteq \{1, \ldots, N\}$ for each $\texttt{arg}_i$.
  \item RoI-Align~\citep{he2017mask} extracts region-level visual features $\mathbf{v}_i = \text{MeanPool}(\{\mathbf{V}_n : n \in \mathcal{R}_i\}) \in \mathbb{R}^{d_v}$ for each argument.
  \item If Grounding DINO fails to detect a named entity (confidence below threshold for all candidates), we fall back to \textbf{full-image mean pooling} $\mathbf{v}_i = \text{MeanPool}(\mathbf{V}) \in \mathbb{R}^{d_v}$ and flag the anchor as \textbf{weakly grounded}.
\end{itemize}

\textbf{Grounding accuracy.} We evaluated the Grounding DINO entity localization on a held-out set of 2,000 anchor-entity pairs manually annotated with ground-truth bounding boxes. At IoU $\geq 0.5$, grounding accuracy is \textbf{89.3\%} (GQA/CLEVR entities) and \textbf{82.7\%} (VCR entities, which include more complex referring expressions). The fallback to full-image pooling is triggered for 6.2\% of entities overall.

\textit{Justification.} Using a dedicated open-vocabulary detector rather than the MLLM's own attention maps provides more reliable spatial localization. Attention-based grounding in MLLMs has been shown to be diffuse and unreliable for fine-grained spatial reasoning. Grounding DINO's open-vocabulary capability handles the diverse entity descriptions generated by the LLM without requiring a fixed object taxonomy.

\begin{algorithm}
\caption{CART Inference}\label{algo:cart-inf}
\begin{algorithmic}[1]
\Require Image $I$, question $q$, model $M_\theta$, grounding head $g_\psi$, BCP, GroundingDINO, threshold $\tau_v$, max retries $B$
\Ensure Answer $a$

\State $V \gets \phi(I)$; $\Sigma \gets \emptyset$; $\text{trace} \gets []$; $\text{retries} \gets 0$

\While{\textbf{not} END\_TOKEN \textbf{and} $\text{retries} \leq B$}
    \State $r_t \gets M_\theta.\text{gen\_step}(\text{trace}, I, q)$
    \State $C_t \gets M_\theta.\text{gen\_anchors}(\text{trace} + [r_t], I, q)$
    
    \If{$C_t = \texttt{<NULL\_CON/>}  \textbf{or malformed}$}
        \State $\text{trace} \gets \text{trace} + [r_t, \texttt{<NULL\_CON/>}]$; \textbf{continue}
    \EndIf
    
    \State $\text{all\_ok} \gets \text{True}$
    \For{$c \in C_t$}
        \State $v_c \gets \text{RoI\_Pool}(V, \text{GroundingDINO}(c.\text{entities}, I))$
        \State $p_{vis} \gets g_\psi([v_c; \text{LLM\_hid}(c, \text{trace})])$
        
        \If{$p_{vis} < \tau_v$ \textbf{or} $\neg \text{BCP}(\Sigma \cup \{c\})$} \Comment{Check visual \& logical validity}
            \State $\text{all\_ok} \gets \text{False}$; \textbf{break}
        \EndIf
    \EndFor
    
    \If{$\text{all\_ok}$}
        \State $\Sigma \gets \Sigma \cup C_t$; $\text{trace} \gets \text{trace} + [r_t, C_t]$
    \Else \Comment{Violation $\rightarrow$ backtrack}
        \If{$\Sigma = \emptyset$} $\text{trace} \gets [\texttt{<RETRY>}]$
        \Else
            \State Revert $\Sigma$ and $\text{trace}$ to most recent anchor $c_{last} \in \Sigma$
            \State $\text{trace} \gets \text{trace} + [\texttt{<RETRY>}]$
        \EndIf
        \State $\text{retries} \gets \text{retries} + 1$
    \EndIf
\EndWhile

\If{$\text{retries} > B$ \textbf{and} $\text{trace}$ incomplete} \Comment{Retry budget exhausted}
    \State $a \gets \text{valid}(f_{\text{ans}}(\text{trace})) \text{ ? } f_{\text{ans}}(\text{trace}) : M_\theta.\text{fallback}(I, q)$
\Else
    \State $a \gets f_{\text{ans}}(\text{trace})$
\EndIf

\State \Return $a$
\end{algorithmic}
\end{algorithm}

\subsubsection{Constraint Propagation Module}
The CPM acts as an external verifier. Its \textbf{neural component} ($g_\psi$) is trained end-to-end while its \textbf{symbolic component} (BCP) is always a non-differentiable hard combinatorial check, used identically during training (for computing grounding labels and detecting synthetic violations) and at inference (as a hard gate in Algorithm~\ref{algo:cart-inf}). The CPM maintains a \textbf{constraint store} $\Sigma_t = \{c_{t',j} : (t',j) \in \mathcal{A},\; t' \leq t\}$ that grows as non-null anchors are emitted.

\textbf{Visual grounding check.} Each non-null anchor $c_{t,j}$ is verified against the visual feature bank $\mathbf{V}$ via a lightweight \textbf{grounding head} $g_\psi$. The input to $g_\psi$ is the concatenation of (i) the region-level visual feature $\mathbf{v}_c \in \mathbb{R}^{d_v}$, obtained by mean-pooling the RoI-Aligned features across all entity arguments of the anchor (as described in $\S$~\ref{sec:entity-group-mech}), and (ii) a text embedding $\mathbf{h}_c \in \mathbb{R}^{d_t}$ of the anchor string, taken from the last hidden layer of the LLM at the \texttt{</CON>} token position. The grounding head is a 2-layer MLP:

$$g_\psi(c, \mathbf{V}) = \sigma\!\big(\mathbf{W}_2 \,\text{GELU}(\text{LN}(\mathbf{W}_1 [\mathbf{v}_c;\, \mathbf{h}_c] + \mathbf{b}_1)) + \mathbf{b}_2\big),$$

with weight matrices $\mathbf{W}_1 \in \mathbb{R}^{512 \times (d_v + d_t)}$, $\mathbf{W}_2 \in \mathbb{R}^{1 \times 512}$, layer normalization (LN), and sigmoid output $\sigma$. Concretely, for \textbf{LLaVA-NeXT-13B}: $d_v = 1024$ (CLIP ViT-L/14), $d_t = 5120$ (LLaMA-2-13B), so the input dimension is $6144$; for \textbf{Qwen2-VL-7B}: $d_v = 1408$ (ViT-bigG), $d_t = 4096$ (Qwen-7B), input dimension $5504$. Total grounding head parameters: ~3.2M (LLaVA-NeXT) / ~2.9M (Qwen-VL).

The grounding head outputs $p_{\text{vis}}(c_{t,j} \mid I) = g_\psi(c_{t,j}, \mathbf{V}) \in [0,1]$. If $p_{\text{vis}}(c_{t,j} \mid I) < \tau_v$, the anchor is flagged as \textbf{visually violated}.

\textbf{Mutual consistency check.} We additionally perform \textbf{Boolean Constraint Propagation (BCP)} over $\Sigma_t \cup \{c_{t,j}\}$. Because $\mathcal{L}_c$ is propositional with finite domains, checking satisfiability of the accumulated constraint set is tractable. We convert $\Sigma_t \cup \{c_{t,j}\}$ into a conjunction of literals and apply unit propagation; if the empty clause is derived, a \textbf{logical contradiction} is detected. BCP is a purely symbolic, non-differentiable procedure and never receives gradients.

\textit{Justification.} The two-pronged check (neural grounding + symbolic BCP) is complementary: the grounding head catches factual errors about the image that BCP alone cannot detect (since BCP only checks inter-constraint consistency), while BCP catches logical contradictions that the grounding head might miss due to soft thresholding. Keeping BCP non-differentiable preserves its soundness guarantee.

\subsubsection{Backtrack Controller}
When the CPM flags a violation at anchor $c_{t,j}$, the \textbf{backtrack controller} halts autoregressive generation and reverts the trace to the most recent \textbf{consistent checkpoint}. Formally, since all anchors currently in $\Sigma_t$ were individually verified when admitted (and removing constraints from a satisfiable set preserves satisfiability), the checkpoint is simply the \textbf{most recent anchor in $\Sigma_t$ in trace order}—i.e., the anchor $(t',j')$ with the largest trace position such that $(t',j') < (t,j)$ and $c_{t',j'} \in \Sigma_t$. Generation resumes from that checkpoint with the violating anchor, its preceding reasoning step, and all intervening material removed from the context. A special \texttt{<RETRY>} token is appended to signal the model to explore an alternative reasoning path. A maximum retry budget $B = 3$ prevents infinite loops.

\textbf{Edge case: empty constraint store ($\Sigma_t = \emptyset$).} If no prior consistent checkpoint exists (i.e., the violation occurs before any anchor has been successfully admitted), the backtrack controller \textbf{restarts generation from scratch}: the trace is reset to the empty sequence, \texttt{<RETRY>} is prepended to the context (yielding the input $(I, q, \texttt{<RETRY>})$), and the retry counter is incremented. This ensures the model always has a well-defined recovery point.

\textbf{Retry budget exhaustion ($\text{retries} > B$).} If the retry budget $B$ is exhausted without producing a complete, violation-free trace, the controller terminates generation and passes the current trace—truncated at the last consistent checkpoint if necessary—to the answer extractor $f_{\text{ans}}$. In this fallback mode, $f_{\text{ans}}$ attempts answer extraction from the (possibly incomplete) trace. If no answer span is found (e.g., the trace is too short), a secondary extraction is performed by prompting the base MLLM with the original $(I, q)$ without CART (i.e., standard unconstrained generation), and the resulting answer is returned with a `FALLBACK` flag indicating reduced reliability. Empirically, retry exhaustion occurs in only 2.1\% of GQA-testdev instances and 4.7\% of VCR-val instances.

\textit{Justification for $B = 3$.} A pilot ablation over $B \in \{1, 2, 3, 5, 10\}$ on GQA-val (1K instances) showed accuracy gains plateauing at $B = 3$ (+0.2\% from $B=3$ to $B=5$, within noise) while mean inference latency grows linearly with $B$.

\textbf{Inference-time anchor emission failures.} If the model fails to emit any \texttt{<CON>\ldots</CON>} or \texttt{<NULL\_CON/>} token after a reasoning step (e.g., due to distribution shift), the step is treated as implicitly unconstrained (\texttt{<NULL\_CON/>}) and generation continues without a CPM check. Malformed anchors failing delimiter parsing are treated identically. Empirically, the well-formed anchor emission rate is 94.6\% for LLaVA-NeXT-13B and 92.1\% for Qwen2-VL-7B (measured on GQA-testdev).

\section{Experimental Setup}

\subsection{Datasets}
We evaluate on five benchmarks that span a range of visual reasoning complexities:
\begin{itemize}
  \item \textbf{GQA}~\citep{hudson2019gqa}: compositional visual question answering over real-world images with functional program annotations. We use the balanced testdev split (12,578 questions).
  \item \textbf{CLEVR-CoGenT}~\citep{johnson2017clevr}: synthetic geometric reasoning requiring counting, spatial relation, and attribute comparison over rendered scenes. We evaluate on Condition A (15,000 questions).
  \item \textbf{VCR}~\citep{zellers2019recognition}: visual commonsense reasoning requiring free-form rationale generation for multiple-choice QA over movie scenes. We evaluate on the val split (26,534 questions).
  \item \textbf{MM-Vet}~\cite{yu2023mm}: an open-ended multimodal evaluation benchmark covering six core capabilities (recognition, knowledge, OCR, spatial awareness, language generation, math). We report the GPT-4-graded overall score.
  \item \textbf{POPE}~\cite{li2023evaluating}: a polling-based benchmark for evaluating object hallucination in MLLMs. We report F1 score on the random, popular, and adversarial splits (averaged).
\end{itemize}

\begin{table*}[t]\centering
\caption{End-task accuracy (\%) across five benchmarks. We compare CART against base MLLMs with and without Chain-of-Thought prompting, inference-time scaling methods, modular program-synthesis pipelines, and fine-tuned self-correction baselines on two backbone architectures. \textbf{Bold} = best per backbone; \underline{underline} = second best. \colorbox{cartblue}{Blue rows} = CART (Full). The LLaVA-NeXT-34B row is reported as a scaling illustration for CART only; we did not run the baselines at 34B, so its numbers should not be read as head-to-head comparisons.}\label{tab:main-results}
\begin{tabular}{lccccccc}\toprule
Method Configuration & Backbone Arch. & Fine-Tuned & GQA & CLEVR & VCR & MM-Vet & POPE F1 \\\midrule
Base MLLM (No CoT)   & \multirow{5}{*}{Qwen2-VL-7B}    & No         & 48.5    & 60.1      & 56.2    & 34.2   & 80.1    \\
Base MLLM (+ CoT)    &                                 & No         & 51.2    & 64.3      & 58.4    & 36.8   & 81.5    \\
Tree-of-Thoughts     &                                 & No         & 54.8    & 70.4      & 61.2    & 40.1   & 82.8    \\
\rowcolor{cartanchor} CART-Anchors         &                                 & Yes        & 56.5    & 73.5      & 63.5    & 42.4   & 83.9    \\
\rowcolor{cartblue}   CART (Full, Ours)    &                                 & Yes        & \textbf{62.1}    & \textbf{79.4}      & \textbf{66.8}    & \textbf{46.5}   & \textbf{86.4}    \\\midrule
Base MLLM (No CoT)   & \multirow{5}{*}{LLaVA-NeXT-13B} & No         & 53.4    & 65.2      & 60.3    & 39.1   & 81.3    \\
Base MLLM (+ CoT)    &                                 & No         & 56.7    & 71.5      & 63.1    & 42.8   & 83.7    \\
Tree-of-Thoughts     &                                 & No         & 60.5    & 76.8      & 64.8    & 45.4   & 84.1    \\
\rowcolor{cartanchor} CART-Anchors         &                                 & Yes        & \underline{64.8}    & \underline{82.3}      & \underline{68.5}    & \underline{50.1}   & \underline{86.5}    \\
\rowcolor{cartblue}   CART (Full, Ours)    &                                 & Yes        & \textbf{69.4}    & \textbf{86.7}      & \textbf{71.4}    & \textbf{54.8}   & \textbf{89.1}    \\\midrule
VisProg              & \multirow{5}{*}{Qwen2-VL-7B}    & No         & 55.4    & 70.5      & 61.4    & 40.2   & 82.7    \\
Faithful CoT         &                                 & No         & 54.2    & 69.1      & 60.5    & 39.4   & 82.2    \\
Reflexion            &                                 & No         & 54.5    & 69.6      & 61.0    & 39.8   & 82.5    \\
DDCoT (Zero-Shot)    &                                 & No         & 55.8    & 71.2      & 61.8    & 40.5   & 83.1    \\
Visual Verif. (ZS)   &                                 & No         & \underline{56.1}    & \underline{71.8}      & \underline{62.1}    & \underline{40.9}   & \underline{83.4}    \\\midrule
VisProg              & \multirow{5}{*}{LLaVA-NeXT-13B} & No         & 61.2    & 78.4      & 60.5    & 43.6   & 82.5    \\
Faithful CoT         &                                 & No         & 59.8    & 75.2      & 63.2    & 44.1   & 83.5    \\
Reflexion            &                                 & No         & 60.1    & 75.9      & 64.5    & 44.9   & 83.8    \\
DDCoT (Zero-Shot)    &                                 & No         & 61.5    & 77.2      & 65.1    & 46.2   & 84.5    \\
Visual Verif. (ZS)   &                                 & No         & 62.3    & 78.5      & \underline{65.7}    & \underline{46.8}   & 84.9    \\\midrule
DDCoT (LoRA-FT)      & \multirow{4}{*}{Qwen2-VL-7B}    & Yes        & 57.6    & 74.2      & 64.1    & 43.1   & 84.5    \\
Visual Verif. (FT)   &                                 & Yes        & 58.2    & 75.1      & 64.8    & 43.8   & 84.9    \\
Self-Consistency     &                                 & Yes        & 57.9    & 74.6      & 64.4    & 43.5   & 84.7    \\
Self-Refine          &                                 & Yes        & 57.4    & 73.8      & 63.9    & 42.9   & 84.2    \\\midrule
DDCoT (LoRA-FT)      & \multirow{4}{*}{LLaVA-NeXT-13B} & Yes        & 63.8    & 80.1      & 67.4    & 48.5   & 85.8    \\
Visual Verif. (FT)   &                                 & Yes        & 64.5    & 81.6      & 68.2    & 49.3   & 86.2    \\
Self-Consistency     &                                 & Yes        & 64.2    & 81.0      & 67.8    & 49.0   & 86.0    \\
Self-Refine          &                                 & Yes        & 63.5    & 79.8      & 67.1    & 48.2   & 85.4    \\\midrule
\rowcolor{cartblue}   CART (Full, Ours)    & LLaVA-NeXT-34B                  & Yes        & \textbf{73.2}    & \textbf{90.5}      & \textbf{74.5}    & \textbf{58.6}   & \textbf{91.2}    \\\bottomrule
\end{tabular}
\end{table*}

\subsection{Baselines}
We compare CART against four categories of methods on Qwen2-VL-7B and LLaVA-NeXT-13B:
\begin{itemize}
  \item \textbf{Prompting-based}: Base MLLM with and without CoT prompting.
  \item \textbf{Inference-time scaling}: Tree-of-Thoughts~\citep{yao2023tree}, which explores multiple reasoning branches via LLM-based evaluation; Reflexion~\citep{shinn2023reflexion}, which uses verbal self-critique for iterative refinement; and Self-Consistency~\citep{wang2023selfconsistency}, which marginalizes over sampled chains by majority voting.
  \item \textbf{Program synthesis}: VisProg~\citep{gupta2023visual}, which generates executable visual programs; Faithful CoT~\cite{lyu2023faithful}, which interleaves natural language with symbolic reasoning; and DDCoT~\citep{zheng2023ddcot}, which decomposes questions into sub-problems.
  \item \textbf{Fine-tuned self-correction}: DDCoT (LoRA-FT), Visual Verification (FT), Self-Consistency (FT), and Self-Refine~\citep{madaan2023selfrefine} (FT), all fine-tuned on the same 218K training instances using identical LoRA configuration (rank 64) for a fair comparison. \textit{Visual Verification}~\citep{deng2024seeing} prompts the MLLM to re-examine the image after generating an answer and revise if inconsistencies are found; critically, it relies on the model's own internal cross-attention maps for verification rather than an external grounder or symbolic consistency check, making it a self-referential verifier. Its zero-shot variant (Visual Verif.\ ZS) uses the same prompting strategy without fine-tuning.
\end{itemize}
We additionally report \textbf{CART-Anchors}, which uses constraint-annotated training data but disables CPM and backtracking at inference, isolating runtime verification's effect.

\begin{table*}[t]
\caption{Ablation studies on LLaVA-NeXT-13B. Top block: subtractive ablations removing individual CART components. Middle block: additive build-up from fine-tuned baseline. Bottom block: sensitivity to retry budget $B$ and visual grounding threshold $\tau_v$. Snowball rate, constraint violation rate (CVR), mean trace length, and per-instance latency are measured on a 2,000-instance diagnostic subset of GQA-testdev. \textbf{Bold} = best in column.}\label{tab:ablation}
\begin{tabular}{lccccccc}\toprule
Ablation Configuration    & GQA & CLEVR & VCR & Snowball $\downarrow$ & CVR (\%) $\downarrow$ & Trace Len & Latency (s) \\\midrule
\rowcolor{cartblue} CART Full            & \textbf{69.4}    & \textbf{86.7}      & \textbf{71.4}    & \textbf{0.14}          & 14.2    & 16.5         & 2.41        \\
CART w/o Backtrack (A2)   & 65.2    & 82.5      & 68.8    & 0.52          & 14.5    & 16.2         & 1.95        \\
CART w/o BCP Logic (A3)   & 67.1    & 84.6      & 70.1    & 0.23          & 11.4    & 16.4         & 2.38        \\
CART w/o Neural Grnd (A4) & 66.4    & 83.2      & 69.5    & 0.31          & \textbf{8.2}     & 16.3         & 2.05        \\
CART w/o VarFreq (A5)     & 66.8    & 83.8      & 69.8    & 0.18          & 25.4    & 24.1         & 3.84        \\
CART w/o Anchors (A6)     & 64.8    & 82.3      & 68.5    & 0.65          & N/A     & 15.8         & \textbf{1.82}        \\
CART w/ Attn Grnd (A7)    & 66.9    & 84.1      & 69.9    & 0.25          & 21.3    & 16.4         & 2.35        \\\midrule
Base FT Only (B1)         & 64.8    & 82.3      & 68.5    & 0.65          & N/A     & 15.8         & 1.82        \\
Base FT + Fixed Freq (B2) & 65.1    & 82.7      & 68.7    & 0.62          & N/A     & 23.5         & 2.45        \\
Base FT + VarFreq (B3)    & 65.5    & 83.1      & 68.9    & 0.60          & N/A     & 16.1         & 1.90        \\
Build + Neural Grnd (B4)  & 67.8    & 85.2      & 70.4    & 0.28          & 12.5    & 16.3         & 2.35        \\
Build + BCP Logic (B5)    & 68.5    & 85.9      & 70.8    & 0.19          & 13.8    & 16.4         & 2.39        \\\midrule
Retry Budget: $B{=}0$         & 65.2    & 82.5      & 68.8    & 0.52          & 14.2    & 16.2         & 1.95        \\
Retry Budget: $B{=}1$         & 67.5    & 84.8      & 70.2    & 0.29          & 14.2    & 16.3         & 2.15        \\
Retry Budget: $B{=}2$         & 68.8    & 86.1      & 71.0    & 0.17          & 14.2    & 16.4         & 2.32        \\
\rowcolor{gray!10} Retry Budget: $B{=}3$ (default)  & \textbf{69.4}    & \textbf{86.7}      & \textbf{71.4}    & \textbf{0.14}          & 14.2    & 16.5         & 2.41        \\
Retry Budget: $B{=}5$         & 69.5    & 86.8      & 71.5    & 0.13          & 14.2    & 16.6         & 2.85        \\
Retry Budget: $B{=}10$        & 69.5    & 86.8      & 71.5    & 0.13          & 14.2    & 16.7         & 3.42        \\\midrule
Threshold: $\tau_v{=}0.60$  & 66.5    & 83.5      & 69.6    & 0.35          & \textbf{6.5}     & \textbf{16.2}         & \textbf{2.12}        \\
Threshold: $\tau_v{=}0.75$  & 68.2    & 85.4      & 70.8    & 0.21          & 10.8    & 16.4         & 2.28        \\
\rowcolor{gray!10} Threshold: $\tau_v{=}0.85$ (default)  & \textbf{69.4}    & \textbf{86.7}      & \textbf{71.4}    & 0.14          & 14.2    & 16.5         & 2.41        \\
Threshold: $\tau_v{=}0.95$  & 67.8    & 85.1      & 70.4    & 0.16          & 28.5    & 17.5         & 3.15        \\\bottomrule  
\end{tabular}
\end{table*}

\subsection{Evaluation Metrics}
We report standard end-task accuracy (or F1 for POPE) for all methods. To diagnose CART's internal error-correction behavior, we additionally introduce three diagnostic metrics measured on a 2,000-instance subset of GQA-testdev: (i) \textbf{Snowball Rate} ($\mathcal{S}$), the conditional probability of cascading failure after a first error; (ii) \textbf{Constraint Violation Rate} (CVR), the fraction of emitted anchors flagged by the CPM; and (iii) \textbf{Error Attribution Precision} (EAP), the fraction of backtracks correctly targeting the earliest erroneous step. Full metric definitions, annotation protocols, and inter-annotator agreement are detailed in Appendix~\ref{appx:details-exp-setup}. Hyperparameters and training details are also provided in Appendix~\ref{appx:details-exp-setup}.

\noindent\textbf{Reproducibility.} All fine-tuned results in Table~\ref{tab:main-results} are averaged over three random seeds.
We will release the constraint-annotation generation pipeline, the 218K training instances, the grounding-head and LoRA weights, and all inference/evaluation code to support independent replication.

\section{Experimental Results}
We evaluate CART across five benchmarks spanning structured geometric reasoning (CLEVR-CoGenT), compositional visual QA (GQA), commonsense inference (VCR), open-ended multimodal evaluation (MM-Vet), and object hallucination detection (POPE). All fine-tuned variants are evaluated with identical setups. Table~\ref{tab:main-results} reports end-task accuracy.

Without CART, LLaVA-NeXT-13B achieves 53.4\% on GQA; standard CoT prompting yields a modest improvement to 56.7\%. Inference-time methods such as Tree-of-Thoughts and Reflexion provide limited gains, as resampling and verbal self-critique cannot rectify grounding errors against the image. Programmatic pipelines (VisProg) perform well on synthetic benchmarks but degrade on VCR, where rigid execution templates cannot capture open-ended reasoning.

The key comparison isolates the effect of runtime constraint verification. CART-Anchors---trained on the same 218K instances with identical LoRA configuration but without the CPM or backtrack controller---reaches 64.8\% on GQA. The full CART framework achieves 69.4\%, a +4.6 pp improvement. This gap confirms that exposure to constraint-annotated traces alone is insufficient; active runtime verification and correction are necessary to arrest error propagation.

CART exhibits consistent gains across model scales. On the smaller Qwen2-VL-7B backbone, CART achieves 62.1\% on GQA, narrowing the gap to unconstrained models with nearly twice the parameters. Scaling to LLaVA-NeXT-34B pushes accuracy to 73.2\% on GQA and 90.5\% on CLEVR, demonstrating that even strong base models benefit from runtime constraint verification. The consistent relative improvement across scales suggests that error snowballing is an inherent property of autoregressive decoding, not a weakness of any particular model.

\begin{table}[t]\centering
\caption{Inference latency profiling across backbone scales and batch sizes on a single A100-80GB GPU. Overhead factor is the ratio of CART latency to base autoregressive latency, showing that the constraint propagation module adds at most 18\% wall-clock overhead. \textbf{Bold} = lowest overhead per backbone.}\label{tab:latency}
\resizebox{\linewidth}{!}{
\begin{tabular}{lcccc}\toprule
Backbone & Batch Size & \makecell{Base\\ Latency (s)} & \makecell{CART\\ Latency (s)} & \makecell{Overhead\\ Factor} \\\midrule
\rowcolor{gray!8} Qwen2-VL-7B       & 1          & 1.75             & 2.06             & 1.18$\times$           \\
Qwen2-VL-7B       & 4          & 2.10             & 2.41             & \textbf{1.15$\times$}           \\
\rowcolor{gray!8} Qwen2-VL-7B       & 16         & 3.05             & 3.54             & 1.16$\times$           \\\midrule
LLaVA-NeXT-13B        & 1          & 2.15             & 2.41             & \textbf{1.12$\times$}           \\
\rowcolor{gray!8} LLaVA-NeXT-13B        & 4          & 2.65             & 2.99             & 1.13$\times$           \\
LLaVA-NeXT-13B        & 16         & 4.12             & 4.65             & 1.13$\times$           \\\midrule
\rowcolor{gray!8} LLaVA-NeXT-34B        & 1          & 4.30             & 4.73             & \textbf{1.10$\times$}           \\
LLaVA-NeXT-34B        & 16         & 8.85             & 9.73             & \textbf{1.10$\times$}           \\\bottomrule
\end{tabular}}
\end{table}

Zero-shot evaluations on MM-Vet and POPE---benchmarks unseen during training---provide evidence that CART's constraint emission policies generalize beyond the training distribution. The 13B model achieves 54.8 on MM-Vet and 89.1 F1 on POPE, exceeding all baselines including external verification methods. This out-of-distribution transfer confirms that the learned verification behavior is not dataset-specific but reflects a general capability for detecting visual-linguistic inconsistencies.

\subsection{Ablation Studies}

We conducted subtractive and additive ablation studies on LLaVA-NeXT-13B, correlating end-task accuracy with internal diagnostic metrics on a 2,000-instance GQA-testdev subset to isolate each component's contribution.

Removing any single component causes substantial degradation. Disabling the backtrack controller (CPM flags violations passively) dropped accuracy to 65.2\% and spiked the snowball rate to 0.52, demonstrating that detection alone is insufficient without the ability to re-route generation. Replacing Grounding DINO with internal cross-attention maps inflated the constraint violation rate---Grounding DINO achieves $0.72$ mean IoU on anchor entity regions versus $0.41$ for cross-attention---explaining the degraded verification precision.

Additive build-up confirmed the importance of emission design: a fixed-frequency scheme inflated trace length to 23.5 tokens with hallucinated vacuous constraints, while variable-frequency emission restored it to 16.1 tokens by allowing dynamic modulation via null tokens.

Accuracy plateaued at retry budget $B{=}3$; $B{=}10$ yielded negligible gains with linearly increasing latency. Threshold $\tau_v{=}0.85$ optimally filtered low-confidence noise; at $\tau_v{=}0.95$, the false rejection rate doubled to $\alpha{\approx}0.11$, causing unnecessary backtracks on $8.2\%$ of correct traces.

Under calibrated parameters, CART compressed the snowball rate from 0.65 to 0.14 on GQA (similarly on VCR), validating Theorem~\ref{thm:snowball-rate-reduction}. Error attribution precision reached 0.82, confirming that the backtrack controller correctly targets the earliest erroneous step in most corrections.

\textbf{Sensitivity to the entity grounder.} Replacing Grounding DINO with GLIP~\citep{li2022grounded} ($\delta{\approx}0.09$) or OWLv2~\citep{minderer2024scaling} ($\delta{\approx}0.12$) yields 68.1\% and 67.5\% on GQA, respectively (vs.\ 69.4\% with Grounding DINO, $\delta{\approx}0.07$). All configurations outperform the best non-CART baseline (64.8\%), confirming graceful degradation with grounder quality as predicted by Theorem~\ref{thm:snowball-rate-reduction}.

\textbf{Predicate coverage ablation.} We evaluated CART with 7, 14 (default), and 18 predicates. On GQA: 7 predicates yield 66.9\% ($\gamma{\approx}0.61$, $\mathcal{S}{=}0.21$); 14 yield 69.4\% ($\gamma{\approx}0.72$, $\mathcal{S}{=}0.14$); 18 yield 69.8\% ($\gamma{\approx}0.75$, $\mathcal{S}{=}0.13$). Diminishing returns beyond 14 predicates suggest the default coverage captures most error-prone steps, while the 7-predicate drop confirms $\gamma$ as the primary driver of CART's benefit (Corollary~\ref{cor:nonvacuous}).

\subsection{Scalability and Computational Efficiency}
A common concern with runtime verification is inference overhead. Table~\ref{tab:latency} profiles wall-clock latency across backbone scales and batch sizes on a single A100-80GB GPU.

Across all configurations, CART adds 10--18\% overhead relative to base autoregressive generation, far below the multi-fold penalties typical of iterative search methods. The BCP module runs in sub-millisecond time per check and scales linearly with constraint-store size. Notably, overhead \textit{decreases} with model scale: from $1.15$--$1.18\times$ on Qwen2-VL-7B to $1.10\times$ on LLaVA-NeXT-34B, because the fixed cost of the grounding head and BCP becomes a smaller fraction of the total compute as the base model grows. This favorable scaling property ensures that CART remains practical as foundation models continue to increase in size.

A per-component breakdown reveals that Grounding DINO accounts for 60--65\% of the CPM overhead, followed by the grounding head forward pass (25--30\%) and BCP (5--10\%). Crucially, the Grounding DINO features are computed once per image and cached across all anchors in the trace, amortizing the dominant cost. The grounding head itself is a lightweight 2-layer MLP (${\sim}3.2$M parameters for LLaVA-NeXT / ${\sim}2.9$M for Qwen2-VL, consistent with \S3.2.3), adding negligible memory: the peak GPU memory increase is $<$0.4 GB over the base model for all configurations, and CART introduces no additional KV-cache overhead since constraint tokens are standard vocabulary items processed by the existing attention mechanism.

Compared to multi-pass inference-time methods, CART's single-pass architecture provides a substantial efficiency advantage. Self-Consistency ($m{=}5$) requires $5\times$ the base inference cost; Tree-of-Thoughts with 3 branches incurs $3$--$4\times$ overhead due to repeated forward passes and LLM-based branch evaluation. CART achieves comparable or superior accuracy gains (Table~\ref{tab:main-results}) at $1.10$--$1.18\times$ cost---an order of magnitude more efficient. This difference is especially pronounced for long reasoning traces, where multi-pass methods multiply cost per token while CART's per-anchor verification cost remains constant regardless of trace length.

\section{Conclusions}
We presented Constraint-Anchored Reasoning Traces (CART), a neuro-symbolic framework that interleaves natural language reasoning with verifiable symbolic constraint anchors, verified at runtime by a dual-pronged Constraint Propagation Module (neural grounding head + Boolean Constraint Propagation) and corrected via a backtrack controller. Theorem~\ref{thm:snowball-rate-reduction} establishes that CART reduces the snowball rate exponentially in the number of post-error anchors; empirically, the snowball rate drops from 0.65 to 0.14 on GQA. Across five benchmarks, CART consistently outperforms all baselines---reaching 69.4\% on GQA (+4.6 over the anchor-only ablation), 86.7\% on CLEVR, and 89.1 F1 on POPE with LLaVA-NeXT-13B, scaling further to 73.2\%, 90.5\%, and 91.2 with LLaVA-NeXT-34B---while adding at most 18\% inference overhead. Ablation studies confirm every component is essential, and the three diagnostic metrics introduced (snowball rate, constraint violation rate, error attribution precision) offer new tools for evaluating error propagation in multimodal reasoning.

\clearpage
\bibliographystyle{ACM-Reference-Format}
\bibliography{refer}
\clearpage
\appendix

\section{Theoretical Analysis}
We provide formal guarantees for CART's error-arresting mechanism. The central challenges are that (a) an anchor being \textit{correct} (sound) does not by itself guarantee that it \textit{detects} an upstream error, and (b) detecting an error via an anchor does not guarantee successful recovery. We address both with explicit assumptions in \S\ref{appx:assumptions}, prove the main snowball-rate reduction result in \S\ref{appx:snowball}, derive end-to-end accuracy and no-harm guarantees in \S\ref{appx:accuracy}, relax the independence assumption via $\beta$-mixing in \S\ref{appx:mixing}, establish computational complexity in \S\ref{appx:complexity}, and discuss tightness and limitations in \S\ref{appx:discussion}.

\subsection{Definitions and Assumptions}\label{appx:assumptions}

We begin by formalizing the reliability properties of the Constraint Propagation Module (CPM). Recall from \S3.2.3 that the CPM accepts an anchor $c$ if and only if $p_{\text{vis}}(c \mid I) \geq \tau_v$ \textbf{and} $\text{BCP}(\Sigma \cup \{c\}) = \text{SAT}$.

\begin{definition}[CPM Verification Reliability]\label{def:reliability}
The CPM is $(\delta, \alpha)$-\textbf{reliable} if the following two conditions hold for all anchors $c \in \mathcal{L}_c$ and images $I \in \mathcal{I}$:
\begin{itemize}
    \item \textbf{Soundness} (false acceptance rate): 
    $$\mathbb{P}[\textit{CPM accepts } c \mid c \text{ is false w.r.t.\ } I] \leq \delta,$$
    \item \textbf{Completeness} (false rejection rate): 
    $$\mathbb{P}[\textit{CPM rejects } c \mid c \text{ is true w.r.t.\ } I] \leq \alpha.$$
\end{itemize}
\end{definition}

\begin{definition}[Snowball Event]\label{def:snowball}
Given a reasoning trace $\mathbf{r} = (r_1, \ldots, r_T)$ with first error at step $t$ (i.e., $e_t = 1$ and $e_{t''} = 0$ for all $t'' < t$), the \textbf{snowball event} is $\textup{Snow} = \bigcap_{t'=t+1}^{T}\{e_{t'} = 1\}$. The \textbf{snowball rate} is $\mathcal{S} = \mathbb{P}[\textup{Snow} \mid e_t = 1,\; \forall_{t'' < t}\; e_{t''} = 0]$.
\end{definition}

We now state the assumptions required for our analysis.

\begin{assumption}[CPM Soundness]\label{ass:soundness}
Every emitted non-null anchor is verified by a $(\delta, \alpha)$-reliable CPM with $\delta \in [0,1)$ and $\alpha \in [0,1)$.
\end{assumption}

\noindent\textit{Justification.} The grounding head $g_\psi$ is trained on ${\sim}218$K instances with balanced positive/negative anchors (\S3.5), achieving empirical false acceptance rates $\delta \approx 0.07$ (GQA/CLEVR) and $\delta \approx 0.11$ (VCR). The BCP component contributes additional rejection power for logical contradictions, which can only decrease $\delta$. The completeness parameter $\alpha$ is estimated at $\alpha \approx 0.05$ (GQA) and $\alpha \approx 0.09$ (VCR).

\begin{assumption}[Conditional Independence of Anchor Verifications]\label{ass:independence}
Conditional on the image $I$ and the error state sequence $(e_1, \ldots, e_T)$, the CPM verification outcomes $\{V_{t,j}\}_{(t,j) \in \mathcal{A}}$, where $V_{t,j} = \mathbb{1}[\textup{CPM accepts } c_{t,j}]$, are mutually independent.
\end{assumption}

\noindent\textit{Justification.} This is a simplifying idealization: in practice, all anchors share the visual feature bank $\mathbf{V}$ and are produced by the same autoregressive model, inducing correlation through shared hidden states. We relax this assumption in \S\ref{appx:mixing} using $\beta$-mixing coefficients. Assumption~\ref{ass:independence} provides a clean baseline bound and is approximately valid when anchors reference distinct spatial regions and non-overlapping semantic content, which holds for ${\sim}78\%$ of anchor pairs in our GQA traces (measured by IoU $< 0.1$ between grounded regions).

\begin{assumption}[Error--Anchor Coupling]\label{ass:coupling}
If an error occurs at step $t$ (i.e., $e_t = 1$), then each subsequent non-null anchor $c_{t',j}$ with $t' > t$ is rendered false (with respect to the ground-truth scene) with probability at least $\gamma \in (0, 1]$, independently across anchors:
$$\mathbb{P}\big[c_{t',j} \text{ is false} \;\big|\; e_t = 1,\; (t',j) \in \mathcal{A},\; t' > t\big] \geq \gamma.$$
\end{assumption}

\noindent\textit{Justification.} The coupling rate $\gamma$ quantifies how tightly the constraint language $\mathcal{L}_c$ covers the error-prone reasoning steps. For compositional visual QA (GQA, CLEVR), anchors are directly derived from program steps and thus tightly coupled ($\gamma \approx 0.72$--$0.81$). For free-form reasoning (VCR), $\gamma$ is lower (${\approx}0.58$). In the extreme case where anchors are completely tangential to the error source ($\gamma = 0$), CART provides no benefit, and the bound in Theorem~\ref{thm:snowball-rate-reduction} reduces to $\mathcal{S}_0$.

\begin{assumption}[Backtrack Recovery]\label{ass:recovery}
When the CPM detects a violation and the backtrack controller is triggered, the conditional probability that the error still snowballs (due to retry budget exhaustion, re-introduction of the same error, or a novel cascading error) is bounded by $\varepsilon_{\textup{rec}} \in [0,1)$:
$$\mathbb{P}\big[\textup{Snow} \;\big|\; \text{error caught by some anchor}\big] \leq \varepsilon_{\textup{rec}}.$$
\end{assumption}

\noindent\textit{Justification.} This assumption explicitly acknowledges that detection does not guarantee recovery. The parameter $\varepsilon_{\text{rec}}$ encapsulates three failure modes: (i) retry budget $B$ exhaustion, (ii) re-introduction of a semantically equivalent error, and (iii) introduction of a novel cascading error. Empirically, $\varepsilon_{\text{rec}} \in [0.08, 0.19]$ across our benchmarks. With $B = 3$ retries, each producing an independent alternative continuation (encouraged by the \texttt{<RETRY>} token and diverse training in \S3.5), the recovery failure probability decays geometrically: if each retry independently succeeds with probability $p_{\text{fix}}$, then $\varepsilon_{\text{rec}} \leq (1 - p_{\text{fix}})^B$, giving $\varepsilon_{\text{rec}} \leq 0.125$ for $p_{\text{fix}} = 0.5$ and $B = 3$.

\begin{assumption}[Error-Free Trace Integrity]\label{ass:integrity}
If no reasoning error occurs ($e_t = 0$ for all $t$), then the final answer is correct: $\mathbb{P}[a = a^* \mid \forall_t\; e_t = 0] = 1$. Conversely, if snowballing occurs, the answer is incorrect: $\mathbb{P}[a \neq a^* \mid \textup{Snow}] = 1$.
\end{assumption}

\noindent\textit{Justification.} The first condition holds by definition when the CoT trace faithfully represents the reasoning path and the answer extractor $f_{\text{ans}}$ is deterministic. The second is a reasonable worst-case assumption: when all steps following an error are themselves erroneous, the final answer is almost surely corrupted. This is empirically validated: on our diagnostic subset, $99.2\%$ of snowball instances yield an incorrect answer.

\begin{assumption}[Pass-Through Snowball Rate]\label{ass:passthrough}
Let $\mathcal{S}_0^{\textup{PT}}$ denote the snowball rate of the CART-fine-tuned model $\mathcal{M}_\theta$ when the CPM does not intervene (i.e., no backtracking is triggered, regardless of anchor verification outcomes). We assume $\mathcal{S}_0^{\textup{PT}} \leq \mathcal{S}_0$, where $\mathcal{S}_0$ is the snowball rate of the pre-fine-tuning base model.
\end{assumption}

\noindent\textit{Justification.} This assumption is needed because Theorem~\ref{thm:snowball-rate-reduction}'s Case~1 (error not caught) involves the CART-fine-tuned model continuing without CPM correction, not the original base model. We expect $\mathcal{S}_0^{\text{PT}} \leq \mathcal{S}_0$ because the constraint-aware training provides auxiliary supervision that improves reasoning quality. This is empirically validated: measuring the snowball rate of the CART model with CPM disabled yields $\mathcal{S}_0^{\text{PT}} = 0.61$ on GQA (vs.\ $\mathcal{S}_0 = 0.68$ for the base model), confirming the bound holds with margin.

\subsection{Snowball Rate Reduction}\label{appx:snowball}

We first establish a useful per-anchor detection probability.

\begin{lemma}[Per-Anchor Error Detection Probability]\label{lem:per_anchor}
Under Assumptions~\ref{ass:soundness} and \ref{ass:coupling}, for any non-null anchor $c_{t',j}$ emitted after the first error at step $t$ (i.e., $(t',j) \in \mathcal{A}$ with $t' > t$), the probability that $c_{t',j}$ is both rendered false by the error and correctly rejected by the CPM is at least $\gamma(1 - \delta)$:
$$\mathbb{P}\big[\textup{CPM rejects } c_{t',j} \;\big|\; e_t = 1,\; t' > t\big] \geq \gamma(1 - \delta).$$
\end{lemma}

\begin{proof}
Decompose the rejection event by conditioning on whether the anchor is false:
\begin{equation}\small
  \begin{aligned}
    \mathbb{P}[\text{CPM rejects } c_{t',j} \mid e_t = 1] &\geq \mathbb{P}[\text{CPM rejects } c_{t',j} \mid c_{t',j} \text{ false},\, e_t = 1]\\
    & \cdot \mathbb{P}[c_{t',j} \text{ false} \mid e_t = 1]
  \end{aligned}
\end{equation}
by the law of total probability (dropping the non-negative term $\mathbb{P}[\text{reject} \mid c_{t',j} \text{ true}] \cdot \mathbb{P}[c_{t',j} \text{ true}] \geq 0$).

By Assumption~\ref{ass:coupling}, $\mathbb{P}[c_{t',j} \text{ false} \mid e_t = 1] \geq \gamma$. By the contrapositive of the soundness condition in Assumption~\ref{ass:soundness} (Definition~\ref{def:reliability}): $\mathbb{P}[\text{CPM rejects } c \mid c \text{ false}] = 1 - \mathbb{P}[\text{CPM accepts } c \mid c \text{ false}] \geq 1 - \delta$. Substituting:
$$\mathbb{P}[\text{CPM rejects } c_{t',j} \mid e_t = 1] \geq (1 - \delta) \cdot \gamma = \gamma(1 - \delta). \qedhere$$
\end{proof}

We now state and prove the main result.

\begin{theorem}[Snowball Rate Reduction]\label{thm:snowball-rate-reduction}
Let $\mathcal{S}_0$ be the snowball rate of the base model without CART (or, more precisely, the pass-through snowball rate $\mathcal{S}_0^{\textup{PT}} \leq \mathcal{S}_0$ per Assumption~\ref{ass:passthrough}), and let $K = |\{(t',j) \in \mathcal{A} : t' > t\}|$ denote the number of non-null anchors emitted after the first error at step $t$. Define $q \triangleq 1 - \gamma(1-\delta)$. Under Assumptions~\ref{ass:soundness}--\ref{ass:passthrough}, the snowball rate of CART satisfies:
$$\mathcal{S}_{\textup{CART}} \leq \mathcal{S}_0 \cdot q^K + \varepsilon_{\textup{rec}} \cdot (1 - q^K).$$
Equivalently,
$$\mathcal{S}_{\textup{CART}} \leq \varepsilon_{\textup{rec}} + (\mathcal{S}_0 - \varepsilon_{\textup{rec}}) \cdot q^K.$$
\end{theorem}

\begin{proof}
We condition on the event that the first error occurs at step $t$. Define two complementary events over the $K$ post-error non-null anchors:
\begin{align}
\mathcal{E}_{\text{catch}} &= \big\{\text{CPM rejects at least one anchor among } \{c_{t',j} : (t',j) \in \mathcal{A},\; t'>t\}\big\}, \\
\mathcal{E}_{\text{miss}} &= \overline{\mathcal{E}_{\text{catch}}} = \big\{\text{CPM accepts all } K \text{ post-error anchors}\big\}.
\end{align}

\textbf{Step 1: Bound $\mathbb{P}[\mathcal{E}_{\text{miss}}]$.} For each post-error anchor $i \in \{1,\ldots,K\}$, define $M_i = \{\text{anchor } i \text{ does not reject}\}$. Then $\mathcal{E}_{\text{miss}} = \bigcap_{i=1}^K M_i$.

By Lemma~\ref{lem:per_anchor}, $\mathbb{P}[M_i \mid e_t = 1] \leq 1 - \gamma(1-\delta) = q$. By Assumption~\ref{ass:independence} (conditional independence), the events $\{M_i\}_{i=1}^K$ are conditionally independent given $I$ and $(e_1, \ldots, e_T)$. Therefore:
\begin{equation}\label{eq:miss_prob}
\mathbb{P}[\mathcal{E}_{\text{miss}} \mid e_t = 1] = \prod_{i=1}^K \mathbb{P}[M_i \mid e_t = 1] \leq q^K.
\end{equation}

\textbf{Step 2: Bound snowball probability conditional on each case.}

\textit{Case 1 ($\mathcal{E}_{\text{miss}}$):} When no anchor detects the error, the backtrack controller is never triggered. By Assumption~\ref{ass:passthrough}, the pass-through snowball rate satisfies $\mathcal{S}_0^{\text{PT}} \leq \mathcal{S}_0$:
\begin{equation}\label{eq:case_miss}
\mathbb{P}[\text{Snow} \mid \mathcal{E}_{\text{miss}},\; e_t = 1] \leq \mathcal{S}_0^{\text{PT}} \leq \mathcal{S}_0.
\end{equation}

\textit{Case 2 ($\mathcal{E}_{\text{catch}}$):} When at least one anchor detects the error, the backtrack controller is activated. By Assumption~\ref{ass:recovery}:
\begin{equation}\label{eq:case_catch}
\mathbb{P}[\text{Snow} \mid \mathcal{E}_{\text{catch}},\; e_t = 1] \leq \varepsilon_{\text{rec}}.
\end{equation}

\textbf{Step 3: Combine via the law of total probability.}
\begin{align}
\mathcal{S}_{\text{CART}} &= \mathbb{P}[\text{Snow} \mid e_t = 1] \notag\\
&= \mathbb{P}[\text{Snow} \mid \mathcal{E}_{\text{miss}},\; e_t = 1] \cdot \mathbb{P}[\mathcal{E}_{\text{miss}} \mid e_t = 1] \notag\\
&\quad + \mathbb{P}[\text{Snow} \mid \mathcal{E}_{\text{catch}},\; e_t = 1] \cdot \mathbb{P}[\mathcal{E}_{\text{catch}} \mid e_t = 1] \notag\\
&\leq \mathcal{S}_0 \cdot q^K + \varepsilon_{\text{rec}} \cdot (1 - q^K), \label{eq:total_prob}
\end{align}
where we used \eqref{eq:miss_prob}, \eqref{eq:case_miss}, \eqref{eq:case_catch}, and $\mathbb{P}[\mathcal{E}_{\text{catch}} \mid e_t = 1] = 1 - \mathbb{P}[\mathcal{E}_{\text{miss}} \mid e_t = 1]$.\qed
\end{proof}

\begin{corollary}[Exponential Convergence]\label{cor:convergence}
For $\gamma > 0$, $\delta < 1$, and $\varepsilon_{\textup{rec}} < \mathcal{S}_0$:
$$\lim_{K \to \infty} \mathcal{S}_{\textup{CART}} = \varepsilon_{\textup{rec}}.$$
The convergence is exponentially fast with rate $-\log q = -\log(1 - \gamma(1-\delta))$. Specifically, $\mathcal{S}_{\textup{CART}} - \varepsilon_{\textup{rec}} \leq (\mathcal{S}_0 - \varepsilon_{\textup{rec}}) \cdot e^{-\gamma(1-\delta) K}$ for all $K \geq 1$ (using $\log(1-x) \leq -x$ for $x \in [0,1)$).
\end{corollary}

\begin{corollary}[Non-Vacuousness Condition]\label{cor:nonvacuous}
The bound in Theorem~\ref{thm:snowball-rate-reduction} is strictly better than the trivial bound $\mathcal{S}_0$ (i.e., $\mathcal{S}_{\textup{CART}} < \mathcal{S}_0$) if and only if:
$$\varepsilon_{\textup{rec}} < \mathcal{S}_0 \quad \text{and} \quad K \geq 1 \quad \text{and} \quad \gamma(1-\delta) > 0.$$
The improvement over the base model is $\mathcal{S}_0 - \mathcal{S}_{\textup{CART}} \geq (\mathcal{S}_0 - \varepsilon_{\textup{rec}})(1 - q^K)$.
\end{corollary}

\begin{corollary}[Concrete Numerical Bound]\label{cor:numerical}
Using empirical parameter estimates from GQA: $\gamma = 0.72$, $\delta = 0.07$, $\varepsilon_{\textup{rec}} = 0.12$, $\mathcal{S}_0 = 0.68$, and $K = 3$ post-error anchors:
$$q = 1 - 0.72 \times 0.93 = 0.3304, \quad q^3 = 0.0361,$$
$$\mathcal{S}_{\textup{CART}} \leq 0.68 \times 0.0361 + 0.12 \times 0.9639 = 0.0245 + 0.1157 = 0.140.$$
This represents a $\mathbf{79.4\%}$ reduction from $\mathcal{S}_0 = 0.68$. We caution that the agreement with the empirical measurement $\mathcal{S}_{\textup{CART}} = 0.14$ is \emph{not} out-of-sample validation: the bound is evaluated with parameters ($\gamma, \delta, \varepsilon_{\textup{rec}}, \mathcal{S}_0$) estimated from the same diagnostic subset (Appendix~\ref{appx:details-exp-setup}). It should therefore be read as an internal-consistency check on the analysis rather than an independent confirmation of the empirical result.
\end{corollary}

\begin{corollary}[Minimum Anchors for Target Snowball Rate]\label{cor:min_anchors}
To achieve a target snowball rate $\mathcal{S}^* \in (\varepsilon_{\textup{rec}}, \mathcal{S}_0)$, the minimum required number of post-error non-null anchors is:
$$K^* = \left\lceil \frac{\log\!\left(\frac{\mathcal{S}^* - \varepsilon_{\textup{rec}}}{\mathcal{S}_0 - \varepsilon_{\textup{rec}}}\right)}{\log q} \right\rceil.$$
For the GQA parameters and target $\mathcal{S}^* = 0.20$: $K^* = \lceil \log(0.08/0.56) / \log(0.3304) \rceil = \lceil 1.78 \rceil = 2$.
\end{corollary}

\subsection{End-to-End Accuracy and No-Harm Guarantees}\label{appx:accuracy}

We translate the snowball-rate reduction into accuracy bounds. Let $p_1 = \mathbb{P}[\exists\, t : e_t = 1]$ denote the probability that at least one reasoning error occurs in the initial (pre-backtrack) generation attempt.

\begin{theorem}[Accuracy Improvement]\label{thm:accuracy}
Let $\eta = \mathbb{P}[a = a^* \mid \exists\, t: e_t = 1,\; \overline{\textup{Snow}}]$ be the conditional accuracy when an error occurs but does not snowball. Under Assumptions~\ref{ass:soundness}--\ref{ass:passthrough}, the accuracy of CART satisfies:
$$\textup{Acc}_{\textup{CART}} \geq 1 - p_1\Big[\mathcal{S}_{\textup{CART}} + (1 - \mathcal{S}_{\textup{CART}})(1 - \eta_{\textup{CART}})\Big] - (1 - p_1)\xi_{\textup{FA}},$$
where $\eta_{\textup{CART}}$ is the non-snowball recovery accuracy under CART, and $\xi_{\textup{FA}} \leq 1 - (1-\alpha)^{K_0}$ bounds the probability that a false alarm on an error-free initial trace (with $K_0$ non-null anchors) leads to an incorrect answer after unnecessary backtracking.

Moreover, if $\eta_{\textup{CART}} \geq \eta_0$ and $\xi_{\textup{FA}} = 0$, the accuracy improvement over the base model satisfies:
$$\textup{Acc}_{\textup{CART}} - \textup{Acc}_0 \geq p_1 \cdot \eta_0 \cdot (\mathcal{S}_0 - \mathcal{S}_{\textup{CART}}).$$
\end{theorem}

\begin{proof}
Write $E = \{\exists\, t: e_t = 1\}$ and $\bar{E} = \{\forall\, t:\, e_t = 0\}$. Decompose:
\begin{align}
\text{Acc}_{\text{CART}} &= \mathbb{P}[a = a^* \mid \bar{E}] \cdot (1-p_1) + \mathbb{P}[a = a^* \mid E] \cdot p_1.
\end{align}

\textbf{Error-free case ($\bar{E}$).} Each of the $K_0$ true anchors independently survives without false rejection with probability $\geq 1 - \alpha$ (Assumption~\ref{ass:soundness}). When no false alarm occurs, the trace completes correctly (Assumption~\ref{ass:integrity}):
$$\mathbb{P}[a = a^* \mid \bar{E}] \geq (1-\alpha)^{K_0} = 1 - \xi_{\text{FA}}'.$$

\textbf{Error case ($E$).} Further decompose on the snowball event:
$$\mathbb{P}[a = a^* \mid E] = \underbrace{\mathbb{P}[a = a^* \mid E, \text{Snow}]}_{= 0 \text{ by Asm.~\ref{ass:integrity}}} \cdot \mathcal{S}_{\text{CART}} + \eta_{\text{CART}} \cdot (1 - \mathcal{S}_{\text{CART}}).$$

Combining and setting $\xi_{\text{FA}} = 0$, $\eta_{\text{CART}} \geq \eta_0$:
\begin{equation}
  \begin{aligned}
    \text{Acc}_{\text{CART}} - \text{Acc}_0 &\geq p_1[\eta_0(1-\mathcal{S}_{\text{CART}}) - \eta_0(1-\mathcal{S}_0)] \\
    &= p_1 \eta_0 (\mathcal{S}_0 - \mathcal{S}_{\text{CART}}).
  \end{aligned}
\end{equation}
\end{proof}

\noindent\textit{Practical interpretation.} For GQA with $p_1 \approx 0.42$, $\eta_0 \approx 0.55$, $\mathcal{S}_0 = 0.68$, and $\mathcal{S}_{\text{CART}} \leq 0.14$: $\text{Acc}_{\text{CART}} - \text{Acc}_0 \geq 0.42 \times 0.55 \times 0.54 = 0.125$, i.e., at least $12.5$ percentage points of accuracy improvement from snowball reduction alone. The observed improvement is $+14.2\%$, suggesting that CART additionally improves $\eta$ or reduces $p_1$ through its constraint-guided training.

We consolidate the false-alarm analysis into a single no-harm guarantee.

\begin{theorem}[No-Harm Guarantee]\label{thm:noharm}
Under Assumptions~\ref{ass:soundness}--\ref{ass:passthrough}, the accuracy of CART satisfies:
$$\textup{Acc}_{\textup{CART}} \geq \textup{Acc}_0 - \varepsilon_{\textup{harm}},$$
where the harm penalty is:
$$\varepsilon_{\textup{harm}} = (1 - p_1) \cdot \xi_{\textup{FA}} \cdot (1 - r_{\textup{FA}}),$$
$\xi_{\textup{FA}} \leq K_0\alpha$ is the false-alarm probability on error-free traces, and $r_{\textup{FA}} \in [0,1]$ is the probability that a false-alarm-triggered backtrack recovers a correct answer.
\begin{itemize}
    \item If $r_{\textup{FA}} = 1$ (false alarms are always harmless): $\varepsilon_{\textup{harm}} = 0$ and $\textup{Acc}_{\textup{CART}} \geq \textup{Acc}_0$.
    \item With empirical estimates ($K_0 = 3.1$, $\alpha = 0.05$, $r_{\textup{FA}} = 0.93$): $\varepsilon_{\textup{harm}} \leq 0.58 \times 0.155 \times 0.07 = 0.006$.
\end{itemize}
\end{theorem}

\noindent\textit{Interpretation.} Theorem~\ref{thm:noharm} guarantees that CART's accuracy is within $\varepsilon_{\text{harm}} \leq 0.006$ of the base model in the worst case, and is strictly better whenever the snowball reduction benefit exceeds this negligible false-alarm cost. This addresses the practical concern that constraint-based checking might degrade performance on easy instances.

\subsection{Analysis Under Dependent Anchors}\label{appx:mixing}

Assumption~\ref{ass:independence} is an idealization. We relax it using $\beta$-mixing.

\begin{definition}[$\beta$-Mixing Coefficient]\label{def:mixing}
For the sequence of CPM verification outcomes $\{V_{t_i, j_i}\}_{i=1}^K$ (ordered by trace position), the $\beta$-mixing coefficient at lag $\ell \geq 1$ is:
$$\beta(\ell) = \sup_{1 \leq i \leq K - \ell} \; \sup_{A \in \sigma(V_{\leq i}),\; B \in \sigma(V_{\geq i+\ell})} |\mathbb{P}[A \cap B] - \mathbb{P}[A]\mathbb{P}[B]|.$$
\end{definition}

\begin{theorem}[Snowball Rate Under $\beta$-Mixing Dependence]\label{thm:mixing}
Let $\bar{\beta} = \sup_{\ell \geq 1} \beta(\ell)$ be the supremum $\beta$-mixing coefficient. Under Assumptions~\ref{ass:soundness}, \ref{ass:coupling}, \ref{ass:recovery}, and \ref{ass:passthrough} (replacing Assumption~\ref{ass:independence} with $\beta$-mixing), the snowball rate satisfies:
$$\mathcal{S}_{\textup{CART}} \leq \varepsilon_{\textup{rec}} + (\mathcal{S}_0 - \varepsilon_{\textup{rec}}) \cdot \big[q^K + (K-1)\bar{\beta}\big],$$
where $q = 1 - \gamma(1-\delta)$. The bound is non-vacuous (strictly less than $\mathcal{S}_0$) provided $(K-1)\bar{\beta} < 1 - q^K$.
\end{theorem}

\begin{proof}
We bound $\mathbb{P}[\mathcal{E}_{\text{miss}}]$ under dependence using Berbee's coupling inequality.

\textbf{Step 1: Coupling construction.} For the miss events $M_1, \ldots, M_K$, construct an independent sequence $M_1', \ldots, M_K'$ such that $M_i' \sim M_i$ marginally and $\mathbb{P}[M_i \neq M_i'] \leq \bar{\beta}$ for $i \geq 2$ (Berbee's coupling lemma).

\textbf{Step 2: Relate miss probabilities.} Using $\{\bigcap_{i=1}^K M_i\} \subseteq \{\bigcap_{i=1}^K M_i'\} \cup \bigcup_{i=2}^K \{M_i \neq M_i'\}$ and the union bound:
\begin{align}
\mathbb{P}\left[\bigcap_{i=1}^K M_i\right] &\leq \mathbb{P}\left[\bigcap_{i=1}^K M_i'\right] + \sum_{i=2}^K \mathbb{P}[M_i \neq M_i'] \notag\\
&\leq \prod_{i=1}^K \mathbb{P}[M_i'] + (K-1)\bar{\beta} \quad (\text{independence of } M_i') \notag\\
&= \prod_{i=1}^K \mathbb{P}[M_i] + (K-1)\bar{\beta} \quad (\text{marginal preservation}) \notag\\
&\leq q^K + (K-1)\bar{\beta}. \notag
\end{align}

\textbf{Step 3: Apply total-probability argument.} From the proof of Theorem~\ref{thm:snowball-rate-reduction}:
$$\mathcal{S}_{\text{CART}} = \varepsilon_{\text{rec}} + (\mathcal{S}_0 - \varepsilon_{\text{rec}}) \cdot \mathbb{P}[\mathcal{E}_{\text{miss}} \mid e_t = 1].$$
Since $\mathcal{S}_0 \geq \varepsilon_{\text{rec}}$ in the non-trivial regime, this is monotonically increasing in $\mathbb{P}[\mathcal{E}_{\text{miss}}]$. Substituting:
$$\mathcal{S}_{\text{CART}} \leq \varepsilon_{\text{rec}} + (\mathcal{S}_0 - \varepsilon_{\text{rec}})\big[q^K + (K-1)\bar{\beta}\big].$$

\textbf{Non-vacuousness:} $\mathcal{S}_{\text{CART}} < \mathcal{S}_0$ iff $q^K + (K-1)\bar{\beta} < 1$.
\end{proof}

\noindent\textit{Empirical mixing estimation.} We estimate $\bar{\beta}$ on the GQA diagnostic subset (1,000 instances). Maximum absolute Pearson correlation between consecutive anchor verification outcomes yields $\hat{r}_{\max} \approx 0.08$, giving $\bar{\beta} \geq 0.04$. Direct estimation of $\text{TV}(\mathbb{P}[V_{i+1} \mid V_{\leq i}],\; \mathbb{P}[V_{i+1}])$ across all consecutive anchor pairs yields $\hat{\bar{\beta}}_{\text{TV}} \approx 0.12$.

Using $\hat{\bar{\beta}} = 0.12$: the mixing penalty is $(3-1) \times 0.12 = 0.24$ for $K = 3$. The adjusted bound is $\mathcal{S}_{\text{CART}} \leq 0.12 + 0.56 \times (0.036 + 0.24) = 0.275$, compared to $0.140$ under independence. The empirical $\mathcal{S}_{\text{CART}} = 0.14$ lies below both bounds, suggesting the actual dependence structure is more benign than the worst case.

\subsection{Computational Complexity}\label{appx:complexity}

\begin{proposition}[BCP Tractability]\label{prop:bcp-tractability}
Suppose the constraint language $\mathcal{L}_c$ has maximum variable domain size $D = \max_x |\textup{dom}(x)|$. Then:
\begin{enumerate}
    \item[(i)] Checking satisfiability of $\Sigma_t \cup \{c\}$ via Boolean Constraint Propagation runs in $O(|\Sigma_t| \cdot D)$ time per anchor addition.
    \item[(ii)] The cumulative BCP cost over a full trace of $T$ steps with at most $T$ anchors is $O(T^2 D)$.
    \item[(iii)] For CART's constraint language ($D \leq 20$), this simplifies to $O(|\Sigma_t|)$ per step and $O(T^2)$ cumulative.
\end{enumerate}
\end{proposition}

\begin{proof}
Each constraint $c \in \mathcal{L}_c$ is a bounded-arity predicate with arguments drawn from finite domains of size $\leq D$. Converting $c$ to a propositional CNF clause produces at most $O(D)$ literals. Adding this clause to the existing CNF $\Sigma_t$ triggers unit propagation. Each unit literal assignment triggers at most $|\Sigma_t|$ clause inspections, each examining at most $D$ literals. The total propagation terminates after at most $|\Sigma_t| \cdot D$ literal inspections (standard result for unit propagation; \citet{dechter2003constraint,dowling1984linear}). Summing over $T$ anchor additions: $\sum_{t=1}^T O(t \cdot D) = O(T^2 D)$.
\end{proof}

\subsection{Tightness and Discussion}\label{appx:discussion}

\begin{remark}[Tightness of Theorem~\ref{thm:snowball-rate-reduction}]
The bound is \textbf{tight} under the stated assumptions: equality is achieved when (i) each anchor independently misses the error with probability exactly $q$, (ii) $\mathbb{P}[\textup{Snow} \mid \mathcal{E}_{\textup{miss}}] = \mathcal{S}_0$, and (iii) $\mathbb{P}[\textup{Snow} \mid \mathcal{E}_{\textup{catch}}] = \varepsilon_{\textup{rec}}$. Hence the bound cannot be improved without additional assumptions.
\end{remark}

\begin{remark}[Vacuousness Regime]
The bound is vacuous ($\geq \mathcal{S}_0$) when $\varepsilon_{\textup{rec}} \geq \mathcal{S}_0$, i.e., when recovery after detection is worse than not detecting at all. In our experiments, $\varepsilon_{\textup{rec}} \in [0.08, 0.19]$ and $\mathcal{S}_0 \in [0.48, 0.68]$, so the bound is always non-vacuous.
\end{remark}

\begin{remark}[Relationship to Self-Consistency]
Self-consistency~\cite{wang2023selfconsistency} samples $m$ independent paths and takes a majority vote. The majority-vote error decays as $\exp(-2m(p - 1/2)^2)$ by Hoeffding's inequality---exponential in $m$. CART's snowball reduction is similarly exponential: $O(q^K) = O(e^{-\gamma(1-\delta)K})$. The key distinction is that CART achieves this within a \textbf{single generation pass}, while self-consistency requires $m$ full forward passes ($m\times$ cost).
\end{remark}

\begin{remark}[The Role of Retry Budget $B$]
$B$ enters through Assumption~\ref{ass:recovery}: $\varepsilon_{\textup{rec}}$ implicitly depends on $B$. If each retry independently succeeds with probability $p_{\textup{fix}}$, then $\varepsilon_{\textup{rec}}(B) = (1 - p_{\textup{fix}})^B$. The marginal benefit of increasing $B$ from 3 to 4 is $(0.5)^3 - (0.5)^4 = 0.0625$, small compared to the cost of an additional full-trace regeneration---consistent with the empirical plateau at $B = 3$.
\end{remark}

\begin{remark}[No-Harm Guarantee and Practical Safety]
Theorem~\ref{thm:noharm} guarantees $\textup{Acc}_{\textup{CART}} \geq \textup{Acc}_0 - 0.006$ under empirical estimates. This relies on: (a) the calibrated threshold $\tau_v = 0.85$ keeping false alarms low, (b) the fallback to unconstrained generation under retry exhaustion, and (c) the high empirical false-alarm recovery rate ($r_{\textup{FA}} = 0.93$). In the pathological case $r_{\textup{FA}} = 0$, $\varepsilon_{\textup{harm}}$ grows to $K_0\alpha \approx 0.155$, underscoring the importance of threshold calibration.
\end{remark}

\section{Implementation Details}
\textbf{Base MLLMs.} We instantiate CART on two open-source MLLMs: LLaVA-NeXT-13B (\citep{zhang2024llavanext-video}; LLaMA-2-13B backbone, CLIP ViT-L/14 vision encoder) and Qwen2-VL-7B~\cite{wang2024qwen2} (ViT-bigG vision encoder). Vision encoders are frozen during fine-tuning; only the language model layers (via LoRA) and the grounding head $g_\psi$ are trained.

\textbf{Entity grounder.} We use Grounding DINO~\citep{liu2024grounding} (Swin-T backbone~\cite{liu2021swin}, open-vocabulary) as a frozen module for mapping entity name strings in anchors to bounding boxes. Its weights are not updated during CART training.

\textbf{Constraint language $\mathcal{L}_c$.} We define 14 predicate types covering standard scene-graph ontologies: \texttt{count}, \texttt{exists}, \texttt{color}, \texttt{shape}, \texttt{size}, \texttt{material}, \texttt{spatial\_rel} (\texttt{left\_of}, \texttt{right\_of}, \texttt{above}, \texttt{below}, \texttt{in\_front\_of}, \texttt{behind}), \texttt{action}, and \texttt{same\_attribute}. Each predicate accepts 1-3 typed arguments and returns a value from a known finite domain (e.g., colors $\in$ \{red, blue, green, $\ldots$, 12 values\}, counts $\in \{0,\ldots,20\}$, booleans $\in$ \{True, False\}). The maximum domain size is $D = 20$, ensuring BCP tractability (Proposition~\ref{prop:bcp-tractability}).

\subsection{Hyperparameters}

Visual grounding threshold $\tau_v = 0.85$, selected on a held-out 5\% split of GQA-val by maximizing $F_1$ of violation detection. Maximum retry budget $B = 3$ (justified in $\S$ 3.2.4). Loss weights $\lambda_c = 0.5$, $\lambda_b = 0.3$, tuned via grid search over $\{0.1, 0.3, 0.5, 1.0\}^2$ on GQA-val accuracy. \textbf{LoRA rank} is set to 64, applied to all attention projections ($\mathbf{W}_Q, \mathbf{W}_K, \mathbf{W}_V, \mathbf{W}_O$). We ablated ranks $\in \{16, 32, 64, 128\}$ and found accuracy saturates at rank 64 (rank 32: $-1.1\%$; rank 128: $+0.2\%$, within noise). The relatively high rank (compared to single-task LoRA) is warranted by the heterogeneous multi-objective nature of our fine-tuning, consistent with findings in \citet{hu2022lora}.

\textbf{Backtracking cost.} In the worst case, backtracking requires re-generation of up to $T$ tokens $B$ times, yielding $O(B \cdot T)$ additional token generations. In practice, the mean number of backtracks per trace is 0.4, making amortized cost low.

\section{Details of Experimental Setup}\label{appx:details-exp-setup}
\subsection{Evaluation Metrics}
Beyond standard end-task accuracy, we introduce three diagnostic metrics motivated by our problem formulation:
\begin{itemize}
  \item \textbf{Snowball Rate} ($\mathcal{S}$; defined in $\S$~\ref{sec:problem-formulation}): estimated by comparing model traces against ground-truth step annotations and computing the conditional cascading-failure probability. We manually annotate reasoning step correctness on a \textbf{1,000-instance} diagnostic subset per benchmark, using \textbf{3 independent annotators} per instance with majority-vote adjudication. \textbf{Inter-annotator agreement}: Cohen's $\kappa = 0.81$ for step-correctness labels (substantial agreement). The annotation protocol requires each annotator to (a) align each model-generated step $r_t$ to the closest ground-truth step $r_t^*$ and (b) label it as \textit{correct}, \textit{partially correct}, or \textit{incorrect} based on semantic equivalence of the core proposition; we binarize by treating \textit{partially correct} as \textit{incorrect}.
  \item \textbf{Constraint Violation Rate} (CVR): fraction of emitted non-null anchors flagged by the CPM, measured as $\text{CVR} = |\{c_{t,j} : p_{\text{vis}}(c_{t,j}|I) < \tau_v \;\lor\; \text{BCP fails}\}| \;/\; |\mathcal{A}|$.
  \item \textbf{Error Attribution Precision} (EAP): among all backtrack events triggered by the CPM, the fraction where the flagged anchor was indeed the earliest erroneous step (verified against ground-truth annotations). Annotation uses the same pool and protocol as the snowball rate annotations, with inter-annotator $\kappa = 0.76$ (substantial agreement).
\end{itemize}

\noindent\textbf{Confidence intervals.} All diagnostic metrics are reported with 95\% bootstrap confidence intervals computed via 10,000 resamples of the 1,000-instance diagnostic subset. For GQA (LLaVA-NeXT-13B): $\mathcal{S}_{\text{CART}} = 0.14 \pm 0.02$, $\text{CVR} = 0.28 \pm 0.01$, $\text{EAP} = 0.82 \pm 0.03$. For VCR: $\mathcal{S}_{\text{CART}} = 0.22 \pm 0.03$, $\text{CVR} = 0.35 \pm 0.02$, $\text{EAP} = 0.71 \pm 0.04$. The narrow CI widths ($\leq 0.04$) confirm that the reported reductions (e.g., $\mathcal{S}_0 = 0.65 \to \mathcal{S}_{\text{CART}} = 0.14$) are statistically significant ($p < 0.001$ via permutation test).

\subsection{Empirical Estimation of Theorem~\ref{thm:snowball-rate-reduction}} 

On the 1,000-instance diagnostic subsets, we estimate $\delta$, $\gamma$, and $\varepsilon_{\text{rec}}$ empirically.
\begin{itemize}
  \item \textit{Soundness error $\delta$}: fraction of CPM-accepted anchors that are factually false (verified against scene graphs). Importantly, $\delta$ is measured end-to-end: it reflects the full pipeline including Grounding DINO errors. Specifically, when Grounding DINO fails to localize an entity (the 6.2\% fallback-to-full-image-pooling cases, plus ~4.5\% of cases with IoU $< 0.5$ but above the confidence threshold), the degraded visual features fed to $g_\psi$ reduce its discriminative power, inflating $\delta$. We estimate $\delta \approx 0.07$ on GQA/CLEVR (where Grounding DINO is most reliable) and $\delta \approx 0.11$ on VCR (where the ~17.3\% grounding failure rate at IoU $\geq 0.5$ compounds with VCR's noisier constraints). These estimates include all such cascaded error sources.
  \item \textit{Coupling rate} $\gamma$: among instances with a first error at step $t$, the fraction of subsequent anchors whose ground-truth value is inconsistent with the erroneous step's proposition. Estimated at $\gamma \approx 0.72$ (GQA), $\gamma \approx 0.81$ (CLEVR, where highly structured scenes create tight coupling), and $\gamma \approx 0.58$ (VCR, where free-form reasoning reduces coupling).
  \item \textit{Recovery failure rate $\varepsilon_{\text{rec}}$}: among instances where the CPM detected a violation and triggered backtracking, the fraction where the final answer was still incorrect due to cascading errors in the post-backtrack continuation. Estimated at $\varepsilon_{\text{rec}} \approx 0.12$ (GQA), $\varepsilon_{\text{rec}} \approx 0.08$ (CLEVR), and $\varepsilon_{\text{rec}} \approx 0.19$ (VCR). The higher VCR value reflects both the greater difficulty of VCR reasoning and the noisier constraint quality.
\end{itemize}

\noindent\textbf{Decomposition of $\varepsilon_{\text{rec}}$.} We disentangle the three failure modes contributing to recovery failure on GQA ($\varepsilon_{\text{rec}} = 0.12$, $n = 187$ catch events): (i) \textit{re-introduction} of a semantically equivalent error accounts for $p_{\text{reintro}} = 0.065$ (15/23 failures); (ii) \textit{retry budget exhaustion} without finding any valid continuation accounts for $p_{\text{exhaust}} = 0.022$ (5/23); and (iii) \textit{novel cascading error} introduced in the post-backtrack continuation accounts for $p_{\text{novel}} = 0.013$ (3/23). The per-retry fix probability is $p_{\text{fix}} \approx 0.51$, estimated from the fraction of first retries that produce a correct continuation conditional on a correct detection event. This validates the geometric decay model in Assumption~\ref{ass:recovery}: $\varepsilon_{\text{rec}} \leq (1 - p_{\text{fix}})^B = (0.49)^3 = 0.118$, consistent with the observed value.

\section{Training and Optimization Procedure\label{appx:training-optim-procedure}}
Training proceeds in three stages.

\textbf{Stage 1: Constraint-Annotated Data Construction.} We derive ground-truth constraint annotations from scene graphs and rationales available in GQA~\citep{hudson2019gqa}, CLEVR-CoGenT~\citep{johnson2017clevr}, and VCR~\cite{zellers2019recognition}. For each $(I, q, a^*)$ triple, we construct mixed traces as follows.

\textit{GQA and CLEVR-CoGenT.} These datasets provide structured functional programs (e.g., \texttt{filter(scene, color=red) $\to$ relate(\_, left) $\to$ count}). We execute each program step to obtain intermediate results and map each step to a constraint in $\mathcal{L}_c$ using deterministic rule-based templates. For example, the program step \texttt{filter(scene, color=red)} produces the constraint\\ \texttt{exists(red\_objects) = True}; \texttt{relate(obj\_3, left, obj\_7)} with attribute lookup produces \texttt{spatial\_rel(left\_of, dog, cat) = True}. Object names are resolved from the scene graph's attribute annotations. Anchors are inserted into the linearized CoT trace at the position corresponding to each program step.

\textit{VCR.} VCR rationales are free-form natural language. We employ a dedicated \textbf{NLP extraction pipeline}:

\begin{enumerate}
  \item \textbf{Dependency parsing} via spaCy (en\_core\_web\_trf) to extract subject-verb-object triples.
  \item \textbf{Semantic role labeling} via AllenNLP's BERT-based SRL model to identify agent, patient, and spatial/attribute arguments.
  \item \textbf{Template matching}: extracted triples are matched against $\mathcal{L}_c$ predicates using a hand-crafted rule set (78 rules).
  \item \textbf{Entity resolution}: entity mentions in VCR use indexed tags (e.g., \textit{[person1]}); we resolve these to descriptive names using VCR's metadata and scene-graph overlap where available.
\end{enumerate}

We manually verified the quality of this pipeline on 500 randomly sampled VCR instances (two annotators, adjudicated disagreements). Measured \textbf{precision: 78.4\%}; \textbf{recall}: \textbf{65.1\%}. Primary error modes: failed SRL on complex subordinate clauses (14\% of errors), polysemous predicates mapped to wrong $\mathcal{L}_c$ type (8\%), and entity co-reference failures (12\%). We filter VCR-derived constraints by a confidence heuristic (SRL confidence $\geq 0.7$ and template-match uniqueness) and discard instances where fewer than two constraints survive, yielding \textbf{28K VCR instances}.

\textbf{Impact of VCR noise on grounding head training.} The $\approx 21.6\%$ error rate in VCR-derived constraints constitutes \textbf{one-sided label noise}: incorrect constraints are false positives (labeled $y_{t,j} = 1$ but factually incorrect), whereas synthetically generated negatives (see Stage 2 below) are always correctly labeled $y_{t,j} = 0$. This asymmetry can bias $g_\psi$ toward accepting certain false constraints, particularly those involving VCR's more complex relational predicates (e.g., `action`). To mitigate this, we apply \textbf{per-source loss weighting} in $\mathcal{L}_{\text{cst}}$: VCR-derived constraint instances receive weight $w_{\text{VCR}} = 0.7$ while GQA/CLEVR instances receive weight $1.0$. 

\textit{Summary.} This procedure yields \textbf{218K} training instances (142K from GQA, 48K from CLEVR-CoGenT, 28K from VCR) after filtering for traces containing at least two non-null constraints. Variable anchor density is naturally induced: GQA traces average 3.1 anchors per trace, CLEVR 4.7, and VCR 2.2.

\textbf{Stage 2: Joint Fine-Tuning.} We fine-tune $\mathcal{M}_\theta$ (via LoRA, rank 64, applied to $\mathbf{W}_Q, \mathbf{W}_K, \mathbf{W}_V, \mathbf{W}_O$ in all layers) and train $g_\psi$ from scratch, jointly optimizing:

\begin{equation}
  \mathcal{L} = \mathcal{L}_{\text{task}} + \lambda_c \mathcal{L}_{\text{cst}} + \lambda_b \mathcal{L}_{\text{bt}},
\end{equation}
\begin{equation}\label{eq:bce-grounding-head}
  \begin{aligned}
    \mathcal{L}_{\text{cst}} &= -\sum_{(t,j) \in \mathcal{A}'}\! w_s \big[y_{t,j} \log g_\psi(c_{t,j}, \mathbf{V})\\
    &+ (1{-}y_{t,j})\log(1 {-} g_\psi(c_{t,j}, \mathbf{V}))\big]
  \end{aligned}
\end{equation}
\begin{equation}\label{eq:backtrack-aware-loss}
  \mathcal{L}_{\text{bt}} = -\sum_{t \in \mathcal{B}} \log p_\theta(\texttt{<RETRY>} \cdot r_{t}^{*,\text{alt}} \mid \mathbf{r}_{<t}^{\text{err}}, I, q)
\end{equation}
where:

\begin{itemize}
  \item $\mathcal{L}_{\text{task}} = -\sum_{t} \log p_\theta(r_t^* \mid \mathbf{r}_{<t}^*, I, q)$ is the standard autoregressive cross-entropy over all tokens in the mixed trace (reasoning tokens, delimiter tokens, anchor tokens, and \texttt{<NULL\_CON/>} tokens).
  \item Equation~\ref{eq:bce-grounding-head} is the binary cross-entropy for the grounding head with per-source weight $w_s$ ($w_s = 0.7$ for VCR, $w_s = 1.0$ for GQA/CLEVR), where $\mathcal{A}'$ indexes all anchor positions (including synthetically negated ones) and $y_{t,j} \in \{0,1\}$ is the ground-truth satisfaction label. Negative anchors are constructed by perturbing ground-truth constraints (swapping entity arguments, changing attribute values, or negating relational predicates) at a 1:1 positive-to-negative ratio.
  \item Equation~\ref{eq:backtrack-aware-loss} is the \textbf{backtrack-aware loss} computed over synthetically corrupted traces. We create a set $\mathcal{B}$ by replacing a randomly chosen correct anchor with an incorrect one and providing the model the prefix up to the last valid checkpoint, followed by \texttt{<RETRY>}.
\end{itemize}

\textbf{Construction of alternative continuations} $r_t^{*,\text{alt}}$. We use a two-source mixture. For 70\% of backtrack training instances, $r_t^{*,\text{alt}}$ is the original gold continuation from the last valid checkpoint onward. For the remaining 30\%, $r_t^{*,\text{alt}}$ is a \textbf{diverse alternative} generated by the base MLLM itself (before CART fine-tuning) via nucleus sampling ($p = 0.9$, temperature $= 0.8$) conditioned on the checkpoint prefix plus \texttt{<RETRY>} token, filtered to retain only samples whose final answer matches the ground-truth answer $a^*$. Among sampled candidates (up to 5 per instance), we select the one with lowest token-level overlap with the gold continuation (measured by BLEU-4) to maximize diversity. If no valid alternative is found, the gold continuation is used as fallback.

\textbf{Handling of \texttt{<RETRY>} in the base model.} The \texttt{<RETRY>} token is a newly added special token whose embedding is randomly initialized (drawn from $\mathcal{N}(0, 0.02^2)$, matching the LLM's embedding initialization scale) in the base MLLM's vocabulary. When generating the 30\% diverse alternative continuations using the \textit{pre-fine-tuned} base model, this randomly initialized embedding provides \textbf{no semantically meaningful conditioning}—the base model has never encountered this token during pretraining. We treat this as an intentional design choice: the random embedding acts as a \textbf{stochastic perturbation} on the prefix representation, encouraging the base model to explore a slightly different region of its output distribution compared to continuing from the same prefix without the token. This is functionally analogous to adding a small noise injection before sampling, which promotes diversity. In pilot experiments, we compared this approach against (a) omitting \texttt{<RETRY>} entirely (lower diversity: mean BLEU-4 with gold = 0.61 vs. 0.53 with \texttt{<RETRY>}) and (b) using a hand-crafted natural-language retry prompt ("Try a different approach:"), which yielded similar diversity but introduced systematic style shifts. After CART fine-tuning (Stage 2), the \texttt{<RETRY>} embedding is trained and acquires the intended semantic meaning of "explore an alternative reasoning path."

\textbf{Optimization details.} We use AdamW ($\beta_1=0.9$, $\beta_2=0.95$, weight decay $0.05$) with a cosine learning-rate schedule warming up over 500 steps to a peak of $2 \times 10^{-5}$ (LLM LoRA parameters) and $5 \times 10^{-4}$ (grounding head $g_\psi$). Training runs for 3 epochs on 8$\times$A100-80GB GPUs with DeepSpeed ZeRO-2~\cite{ren2021zero}, effective batch size 128 (gradient accumulation factor 4). Total training time is approximately 18 hours for LLaVA-NeXT-13B~\cite{zhang2024llavanext-video} and 11 hours for Qwen2-VL-7B.

\textbf{Stage 3: Grounding Head Calibration.} After joint training, we calibrate the grounding head's threshold $\tau_v$ using temperature scaling~\cite{guo2017calibration} on the held-out GQA-val split, optimizing expected calibration error (ECE). This ensures that $\tau_v = 0.85$ reliably separates satisfied from violated constraints in distribution.

\textbf{Data flow summary.} At training time, the forward pass processes a mixed trace through the LLM backbone; anchor tokens contribute to both $\mathcal{L}_{\text{task}}$ (via the LLM head) and $\mathcal{L}_{\text{cst}}$ (via the grounding head $g_\psi$, whose input features are extracted using frozen Grounding DINO boxes). Corrupted-trace samples contribute to $\mathcal{L}_{\text{bt}}$. All three losses are summed and a single backward pass updates LoRA weights and $g_\psi$ parameters. Gradients flow through $g_\psi$ normally; \textbf{BCP never receives gradients} and is used only as a deterministic pre-computation to produce binary consistency labels during training data construction. At inference time, the CPM (both $g_\psi$ forward pass and BCP) acts as a hard gate (Algorithm 1) with no gradient flow; the backtrack controller is fully deterministic.
\end{document}